\documentclass[11pt]{article}

\usepackage[english]{babel}
\usepackage[T1]{fontenc}
\usepackage[a4paper,top=3cm,bottom=2cm,left=3cm,right=3cm,marginparwidth=1.75cm]{geometry}

\usepackage[export]{adjustbox}

\usepackage{bibspacing}

\usepackage{authblk}
\usepackage{wrapfig}
\usepackage{subcaption}
\usepackage[cal=boondox]{mathalfa}
\usepackage{marginnote}

\usepackage{mathrsfs}
\usepackage{amsmath}
\usepackage{amssymb}
\newcommand{\rom}[1]{%
  \textup{\uppercase\expandafter{\romannumeral#1}}%
}

\usepackage[utf8]{inputenc}
\usepackage{amsthm}
\usepackage{thmtools, thm-restate}
\declaretheorem{claim}
\declaretheorem{theorem}
\declaretheorem{lemma}

\newtheorem*{theorem*}{Theorem}

\usepackage{booktabs, adjustbox}
\usepackage{multirow}
\usepackage{threeparttable}
\usepackage{graphicx}
\usepackage{array}
\usepackage{tabularx}

\usepackage[colorlinks=true, allcolors=blue, pdfencoding=auto]{hyperref}

\newcommand{\bx}{{\pmb x}}
\newcommand{\by}{{\pmb y}}

\newcommand{\bw}{{\pmb w}}

\newcommand{\be}{{\pmb e}}

\newcommand{\bu}{{\pmb u}}
\newcommand{\bv}{{\pmb v}}

\newcommand{\bdelta}{{\pmb \delta}}

\newcommand{\HH}{{\mathcal H}}

\let\oldphi\phi
\renewcommand\phi{\operatorname{\oldphi}}

\newcommand{\tr}{\mathsf{tr}}

\newcommand\defeq{\triangleq}
\newcommand\norm[1]{\left\lVert#1\right\rVert}

\usepackage{mathtools}

\usepackage{mathrsfs}
\DeclareMathAlphabet{\mathpzc}{OT1}{pzc}{m}{it}
\newcommand\iprod[2]{\langle #1, #2 \rangle}

\newcommand\expect[1]{\mathbb{E}{\left\lbrack#1\right\rbrack}}

\newcommand\expec[2]{\mathbb{E}_{#1}{\left\lbrack#2\right\rbrack}}
\newcommand\expecta[1]{\mathbb{E}{#1}}

\theoremstyle{definition}

\usepackage{mathtools}
\usepackage{listings}
\usepackage{tikz}
\usepackage[nomessages]{fp}
\usepackage{adjustbox}
\usepackage{outlines}
\usepackage{times}
\usepackage{graphicx,color}
\usepackage{array,float}
\usepackage{url}
\usepackage{amstext,amssymb,amsmath}
\usepackage{hyphenat}
\usepackage{amsthm}
\usepackage{verbatim}
\usepackage{bm}
\usepackage{paralist}
\usepackage{ulem}
\newtheorem{assumption}{Assumption}
\newcommand{\ip}[2]{\langle #1,#2\rangle}

\newcommand{\cL}{\mathcal{L}}

\newcommand{\ex}[2]{\underset{#1}{\mathbb{E}}\left[ #2 \right]}

\newcommand{\argmin}{\mathop{\mathrm{argmin}}}

\newcommand{\bP}{\mathbf{P}}
\newcommand{\bQ}{\mathbf{Q}}

\newcommand{\re}{\mathbb{R}}

\usepackage{etoolbox}

\newcommand{\surround}[2][r]%
  {\ifstrequal{#1}{round}%
    {\left( #2 \right)}%
    {\ifstrequal{#1}{square}%
      {\left[ #2 \right]}%
      {\ifstrequal{#1}{curly}%
        {\left\{ #2 \right\}}%
        {\ifstrequal{#1}{angle}%
          {\left\langle #2 \right\rangle}%
          {\ifstrequal{#1}{|}%
            {\left\lvert #2 \right\rvert}%
            {\ifstrequal{#1}{||}%
              {\left\lVert #2 \right\rVert}%
              {\ifstrequal{#1}{floor}%
                {\left\lfloor #2 \right\rfloor}%
                {\ifstrequal{#1}{ceil}%
                  {\left\lceil #2 \right\rceil}%
                  {\ifstrequal{#1}{.}%
                    {\left. #2 \right.}%
                    {\left( #2 \right)}%
                  }%
                }%
              }%
            }%
          }%
        }%
      }%
    }%
  }

\title{The Power of Interpolation: 
Understanding the Effectiveness of SGD in Modern Over-parametrized Learning}
\author[]{Siyuan Ma}
\author[]{Raef Bassily}
\author[]{Mikhail Belkin}

\affil{Department of Computer Science and Engineering}
\affil{The Ohio State University}

\affil{
\textit{\{ma.588, bassily.1\}@osu.edu},
\textit{mbelkin@cse.ohio-state.edu}}
\date{}

\begin{document}

\maketitle
\begin{abstract}
Stochastic Gradient Descent (SGD) with small mini-batch is a key component in modern  large-scale machine learning. However, its efficiency has not been easy to analyze as most theoretical results require adaptive rates and show convergence rates far slower than that for gradient descent, making computational comparisons difficult.

In this paper we aim to formally explain the phenomenon of fast convergence of SGD observed in modern machine learning. The key observation is that most modern learning architectures are over-parametrized and are trained to interpolate the data by driving the empirical loss (classification and regression) close to zero. While it is still unclear why these interpolated solutions perform well on test data, we show that these regimes allow for fast convergence of SGD, comparable in number of iterations to full gradient descent.

For convex loss functions we obtain an exponential  convergence bound for {\it mini-batch} SGD  parallel to that for full gradient descent.
We show that there is a 
critical batch size $m^*$ such 
that:
\begin{itemize}
\item SGD iteration with mini-batch size $m\leq m^*$ is  nearly equivalent to $m$ iterations
of mini-batch size $1$ (\emph{linear scaling regime}).
\item  SGD iteration with mini-batch $m> m^*$ is nearly equivalent to a full gradient descent iteration (\emph{saturation regime}).
\end{itemize}

Moreover, for the quadratic loss, we derive explicit expressions for the optimal mini-batch and step size and explicitly characterize the two regimes above.
The critical mini-batch size can be viewed as the limit for effective mini-batch parallelization. It is also nearly independent of the data size, implying $O(n)$ acceleration over GD per unit of computation.
We give experimental evidence on real data which closely follows
our theoretical analyses.

Finally, we show how  our results fit in the recent developments in training deep neural networks and discuss connections to adaptive rates for SGD and variance reduction.
\end{abstract}



\section{Introduction}

Most machine learning techniques for supervised learning are based on Empirical Loss Minimization (ERM), i.e., minimizing the loss  $\cL(\bw)\triangleq \frac{1}{n}\sum_{i=1}^n\ell_i(\bw)$   over some parametrized space of functions $f_\bw$. Here $\ell_i(\bw) =L(f_\bw(\bx_i),y_i)$, where $(\bx_i,y_i)$ are the data and $L$ could, for example, be the square loss $L(f_\bw(\bx),y)= (f_\bw(\bx) - y)^2$.

In recent years, Stochastic Gradient Descent (SGD) with a small mini-batch size has become the backbone of machine learning, used in nearly all large-scale applications of machine learning methods, notably in conjunction with deep neural networks. Mini-batch SGD is a first order method which, instead of computing the full gradient of $\cL(\bw)$, computes 
the gradient with respect to a certain subset of the data points, often chosen sequentially. 
In practice small mini-batch SGD 
consistently  outperforms  full gradient descent (GD) by a large factor in terms of the computations required to achieve certain accuracy.  However, the theoretical evidence has been mixed. 
While SGD needs less computations per iteration, most analyses suggest that it requires  adaptive step sizes and has the rate of convergence that is far slower than that of GD, making computational efficiency comparisons difficult. 

In this paper, we explain the reasons for the effectiveness of SGD by taking a different perspective.  We note that most of modern machine learning, especially deep learning, relies on classifiers which are trained to achieve near zero 
classification and regression losses on the training data.  Indeed, the goal of achieving near-perfect fit on the training set is stated explicitly by the practitioners as a best practice in supervised learning\footnote{Potentially using regularization at a later stage, after a near-perfect fit is achieved.}, see, e.g., the tutorial~\cite{russ17}.  
The ability to achieve near-zero loss is provided by over-parametrization. The number of parameters for most deep architectures is very large and often exceeds by far the size of the datasets used for training (see, e.g.,~\cite{canziani2016analysis} for a summary of different architectures). There is significant theoretical and empirical evidence that in such over-parametrized systems most or all local minima are also global and hence correspond to the regime where the output of the learning algorithm matches the labels exactly \cite{gupta2015deep, chaudhari2016entropy, zhang2016understanding,huang2016densely,sagun2017empirical,bartlett2017spectrally}. 
Since continuous loss functions are typically used for training, the resulting function {\it interpolates} the data\footnote{Most of these architectures should be able to achieve perfect interpolation, $f_{\bw^*}(\bx_i)= y_i$. In practice, of course, it is not possible even for linear systems due to the computational and numerical limitations.}, i.e., $f_{\bw^*}(\bx_i)\approx y_i$. 

While we do not yet understand why these interpolated classifiers generalize so well to unseen data, there is ample empirical evidence for their excellent  generalization performance in deep neural networks~\cite{gupta2015deep, chaudhari2016entropy, zhang2016understanding,huang2016densely,sagun2017empirical}, kernel machines~\cite{belkin2018understand} and boosting~\cite{schapire1998}. 
In this paper we look at the significant computational implications of this startling phenomenon for  stochastic gradient descent. 

Our first key observation is that in the interpolated regime SGD with fixed step size converges exponentially fast for convex loss functions. The results showing exponential convergence of SGD when the optimal solution minimizes  the loss function at each point go back to the Kaczmarz method~\cite{kaczmarz1937angenaherte} for quadratic functions, more recently analyzed in~\cite{strohmer2009randomized}. For the general convex case, it was first proved in ~\cite{moulines2011non}. The rate was later improved in~\cite{needell2014stochastic}.
However, to the best of our knowledge, exponential 
convergence in that regime has not been connected to over-parametrization and interpolation in modern machine learning. Still, exponential convergence by itself does not allow us to make  any comparisons between the computational efficiency of SGD with different mini-batch sizes and full gradient descent, as the existing results do not depend on the mini-batch size $m$. This dependence is crucial for understanding SGD, as small mini-batch SGD seems to dramatically outperform full gradient descent in nearly all applications.  Motivated by this, in this paper we provide an explanation for the empirically observed efficiency of small mini-batch SGD.  We provide a detailed analysis for the rates of convergence and computational efficiency for different mini-batch sizes and a discussion of its implications in the context of modern machine learning. 

We first analyze convergence of mini-batch SGD for convex loss functions as a function of the batch size $m$. We show that
there is a critical mini-batch size $m^*$ that is nearly independent on $n$, such that the following holds:
\begin{enumerate}
\item  
(linear scaling) One  SGD iteration with mini-batch of size $m\le m^*$ is  equivalent to $m$ iterations  of mini-batch of size one up to a multiplicative constant close to $1$.
\item
(saturation) One SGD iterations with a mini-batch of size $m > m^*$ is nearly (up to a small constant)  as effective as one iteration of full gradient descent.
\end{enumerate}

We see that the critical mini-batch size $m^*$ can be viewed as the limit for the effective parallelization of  mini-batch computations. If an iteration with mini-batch of size $m\le m^*$ can be computed in parallel, it is nearly equivalent to $m$ sequential steps with mini-batch of size $1$. For $m>m^*$ parallel computation has limited added value. 

\begin{wrapfigure}{r}{0.5\textwidth}
  \centering
  \includegraphics[width=0.5\textwidth]{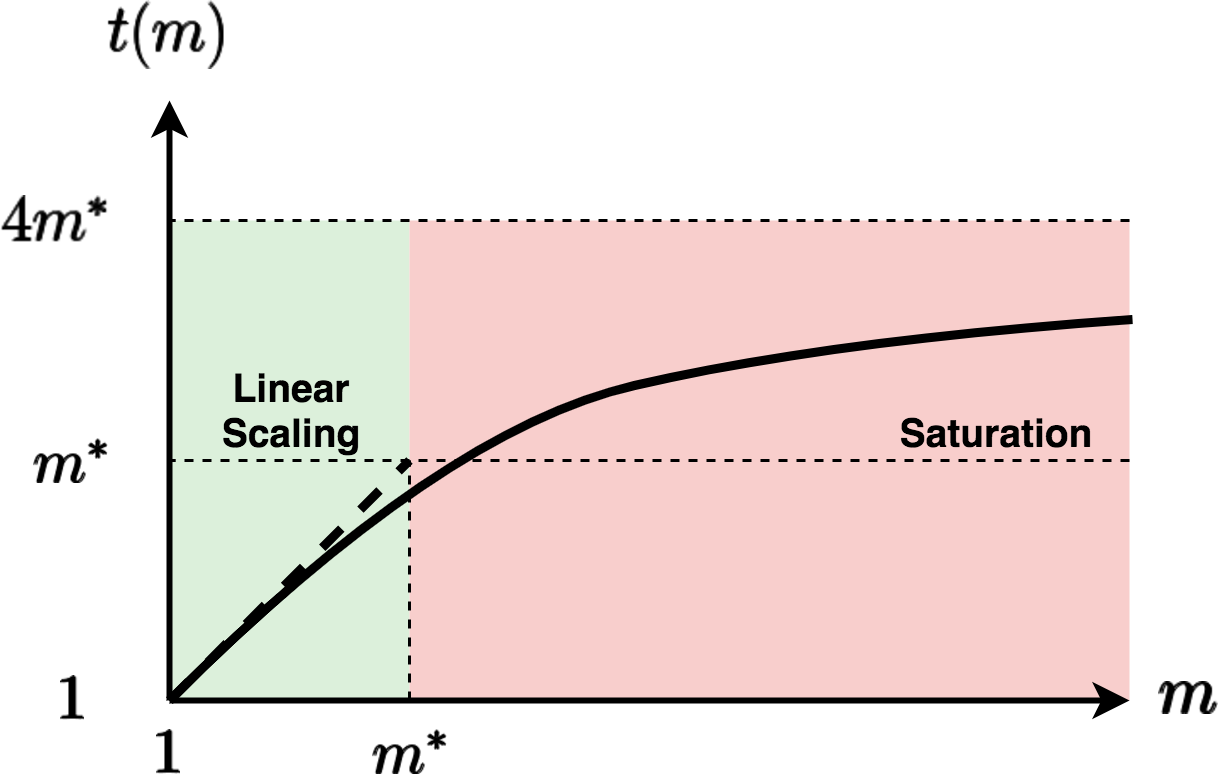}
	\caption{$t(m)$ iterations with batch size $1$ (the $y$ axis) equivalent to one iteration with batch size $m$ (the $x$ axis) for convergence.}
  \label{fig:optimal-region-0}
  \vspace{-5mm}
\end{wrapfigure}

Next, for the quadratic loss function, we obtain a sharp characterization of these regimes based on an explicit derivation of optimal step size as a function of $m$. In particular, in this case we show that the critical mini-batch size is given by 
$m^*= \frac{\max_{i=1}^n\{\norm{\bx_i}^2\}}{\lambda_1 (H)}$
, where $H$ is the Hessian at the minimizer and $\lambda_1$ is its spectral norm. 

Our result shows that $m^*$ is nearly independent of the data size $n$ (depending only on the properties of the Hessian). Thus SGD with mini-batch size $m^*$ (typically a small constant)  gives essentially the same convergence per iteration as full gradient descent, implying acceleration by a factor of  $O(n)$  over GD per unit of computation. 

We also show that a mini-batch of size one is optimal in terms of computations required to achieve a given error. 
Our theoretical results are based on upper bounds which we show to be tight in the quadratic case and nearly tight in the general convex case.


There have been work on understanding the interplay between the mini-batch size and computational efficiency, including~\cite{takac2013mini,li2014efficient,yin2018gradient} in the standard non-interpolated regime. However, in that setting the issue of bridging the exponential convergence of full GD and the much slower convergence rates of mini-batch SGD is harder to resolve, requiring extra components, such as tail averaging~\cite{jain2016parallelizing} (for quadratic loss).

We provide experimental evidence corroborating this on real data.  In particular, we demonstrate the regimes of linear scaling and saturation and also  show
that on real data $m^*$ is in line with our estimate. It is typically several orders of magnitude smaller than the data size $n$ implying a computational advantage of at least $10^3$ factor over full gradient descent in realistic scenarios in the over-parametrized (or fully parametrized) setting.  
We believe this sheds light on the impressive effectiveness of SGD observed in many real-world situation and is the reason why full gradient descent is rarely, if ever, used. In particular, the
``linear scaling rule'' recently used in deep convolutional networks~\cite{krizhevsky2014one, goyal2017accurate, you2017scaling, smith2017don} is consistent with our theoretical analyses.

The rest of the paper is organized as follows:

In Section~\ref{sec:warmup}, we analyze the fast convergence of mini-batch SGD and discuss some implications for the variance reduction techniques. It turns out that in the interpolated regime, simple SGD with constant step size is equally or more effective than the more complex variance reduction methods. 

Section~\ref{sec:fast} contains the analysis of the special case of quadratic losses, where we obtain optimal convergence rates of mini-batch SGD, and derive the optimal step size as a function of the mini-batch size. We also analyze the computational efficiency as a function of the mini-batch size.

In Section~\ref{sec:expr} we provide experimental evidence using several datasets. We show that the experimental results correspond closely to the behavior predicted by our bounds. We also briefly discuss the connection to the linear scaling rule in neural networks.

\section{Preliminaries}

Before we start our technical discussion, we briefly overview some standard notions in convex analysis. Here, we will focus on differentiable convex functions, however, the definitions below extend to general functions simply by replacing the gradient of the function at a given point by the set of all sub-gradients at that point. In fact, since in this paper we only consider smooth functions, differentiability is directly implied.

\begin{itemize}

\item A differentiable function $\ell:\mathbb{R}^d\rightarrow \mathbb{R}$ is \emph{convex} on $\mathbb{R}^d$ if, for all $\bw, \bv\in
  \mathbb{R}^d$, we have $\ell(\bv) \geq \ell(\bw) + \ip{\nabla \ell(\bw)}{\bv-\bw}$.

\item   Let $\beta >0$. A differentiable function $\ell:\mathbb{R}^d\rightarrow \mathbb{R}$ is \emph{$\beta$-smooth} on $\mathbb{R}^d$ if, for all $\bw, \bv\in \mathbb{R}^d$, we have  
  $\ell(\bv) \leq \ell(\bw) + \ip{ \nabla \ell(\bw) }{\bv-\bw} + \frac{\beta}{2}\norm{\bv-\bw}^2,$
 where $\nabla \ell(\bw)$ denotes the gradient of $\ell$ at $\bw$.
\item  Let $\alpha>0$. A differentiable function $\ell:\mathbb{R}^d\rightarrow \mathbb{R}$ is \emph{$\alpha$-strongly convex} on $\mathbb{R}^d$ if, for all $\bw, \bv\in
  \mathbb{R}^d$, we have $\ell(\bv) \geq \ell(\bw) + \ip{\nabla \ell(\bw)}{\bv-\bw} +
  \frac{\alpha}{2}\norm{\bv-\bw}^2$.
(Clearly, for any $\alpha \geq 0$, $\alpha$-strong convexity implies convexity).
\end{itemize}

The problem of unconstrained Empirical Risk Minimization (ERM) can be described as follows: Given a set of $n$ loss functions $\ell_i:\mathbb{R}^d\rightarrow \mathbb{R},~ i\in \{1, \ldots, n\}$, the goal is to minimize the empirical loss function defined as 
$$\cL(\bw)\triangleq \frac{1}{n}\sum_{i=1}^n \ell_i(\bw),~ \bw\in\mathbb{R}^d.$$ In particular, we want to find a minimizer $\bw^*\triangleq \arg\min_{\bw\in\mathbb{R}^d} \cL(\bw)$. In the context of supervised learning, given a training set $\{(\bx_i, y_i): 1\leq i\leq n\}$ of $n$ (feature vector, target) pairs, one can think of $\ell_i(\bw)$ as the cost incurred in choosing a parameter vector $\bw$ to fit the data point $(\bx_i, y_i)$. In particular, in this context, minimizing $\cL$ over $\bw\in\mathbb{R}^d$ is equivalent to minimizing $\cL$ over a parameterized space of functions $\{f_\bw: \bw\in\mathbb{R}^d\}$, where each $f_\bw$ maps a feature vector $\bx$ to a target $y$. Thus, in this case, for each $i$, $\ell_i(\bw)$ can be written as $L(f_\bw(\bx_i), y_i)$ where $L$ is some cost function that represents how far is $f_\bw(\bx_i)$ from $y_i$, for example, $L(\cdot, \cdot)$ could be the squared loss $L(f_\bw(\bx), y)=\left(f_\bw(\bx)- y\right)^2$. 

\section{Interpolation and Fast SGD: Convex Loss}
\label{sec:warmup}

We consider a standard setting of ERM where for all $1\leq i\leq n$, $\ell_i$ is non-negative, $\beta$-smooth and convex. Moreover, $\cL(\bw)= \frac{1}{n}\sum_{i=1}^n\ell_i(\bw)$ is 
$\lambda$-smooth and
$\alpha$-strongly convex.
It is easy to see that $\beta \ge \lambda$. %
This setting is naturally satisfied in many problems, e.g., in least-squares linear regression with full rank sample covariance matrix. 

Next, we state our key assumption in this work. This assumption describes the interpolation setting, which is aligned with what we usually observe in over-parametrized settings in modern machine learning.

\begin{assumption}[Interpolation]\label{assump:interp} 
Let $\bw^* \in \argmin_{\bw\in\re^d}\cL(\bw)$. Then, for all $1\leq i\leq n$, \mbox{$\ell_i(\bw^*)=0$.}
\end{assumption}

Note that instead of assuming that $\ell_i(\bw^*)=0$, it suffices to assume that $\bw^*$ is  the minimizer of all $\ell_i$.  By subtracting from each $\ell_i$ the offset $\ell_i(\bw^*)$, we get an equivalent minimization problem where the new losses are all non-negative, and are all zero at $\bw^*$.

Consider the SGD algorithm that starts at an arbitrary $\bw_0\in\re^d$, and at each iteration $t$ makes an update with a constant step size $\eta$:
\begin{align}
\bw_{t+1}&=\bw_t -\eta \cdot \nabla \left\{ \frac{1}{m}\sum_{j=1}^m  \ell_{i_t^{(j)}}(\bw_t)
\right\}
\label{main-update-step}
\end{align}
where $m$ is the size of a mini-batch of data points whose indices $\{i_t^{(1)}, \ldots, i_t^{(m)}\}$ are drawn uniformly with replacement at each iteration $t$ from $\{1, \ldots, n\}$.

The theorem below shows exponential convergence for mini-batch SGD in the interpolated regime. 

\begin{theorem}\label{thm:convex}
For the setting described above and under Assumption~\ref{assump:interp}, for any mini-batch size $m \in \mathbb{N}$, the SGD iteration~(\ref{main-update-step}) with constant step size
$\eta^*(m) \defeq \frac{m}{\beta + \lambda (m - 1)}$
gives the following guarantee
\begin{equation}\label{eq:theorem1}
\begin{split}
\ex{\bw_t}{\cL(\bw_t)}
&\leq \frac{\lambda}{2}
(1 - \eta^*(m) \cdot \alpha)^t
\norm{\bw_{0} - \bw^*}^2
\end{split}
\end{equation}
\end{theorem}
\begin{proof}
By the $\lambda$-smoothness of $\cL$, we have 
\begin{equation}\label{eq:t1-smooth}
\cL(\bw_t)\leq \cL(\bw^*)+\ip{\nabla \cL(\bw^*)}{\bw_t-\bw^*}+\frac{\lambda}{2}\norm{\bw_t-\bw^*}^2=\frac{\lambda}{2}\norm{\bw_t-\bw^*}^2
\end{equation}
Then we prove inequality~(\ref{eq:theorem1}) by showing that
\begin{equation*}
\ex{\bw_t}{\norm{\bw_{t} - \bw^*}^2}
\leq
(1 - \eta^*(m) \cdot \alpha) \norm{\bw_{t-1} - \bw^*}^2
\end{equation*}

For simplicity, we first rewrite the SGD update~(\ref{main-update-step}) with error
$\bdelta_t \defeq \bw_t - \bw^*$
and mini-batch empirical loss
$L_{m,t}(\bw) \defeq \frac{1}{m}
\sum_{i=1}^m \ell_{t_i}(\bw)$,
$$
\bdelta_{t} = \bdelta_{t-1} - \eta \nabla L_{m,t}(\bw_{t-1})
$$

As the key of this proof is to upper bound
$\expecta{\norm{\bdelta_{t}}^2}$
with $\expecta{\norm{\bdelta_{t-1}}^2}$,
we start by expanding $\expecta{\norm{\bdelta_{t}}^2}$ using the above iteration.
\begin{equation}\label{eq:convex-1}
\expec{t_1, \ldots, t_m}{\norm{\bdelta_{t}}^2}
= \expec{t_1, \ldots, t_m}
{\norm{\bdelta_{t-1}}^2
- 2\eta \iprod{\bdelta_{t-1}}{\nabla L_{m,t}(\bw_{t-1})}
+ \eta^2 \norm{\nabla L_{m,t}(\bw_{t-1})}^2 }
\end{equation}

Applying the expectation to the inner product and using the $\alpha$-strong convexity of $\cL$, we have
\begin{equation}\label{eq:convex-2}
\expec{t_1, \ldots, t_m}{\iprod{\bdelta_{t-1}}{\nabla L_{m,t}(\bw_{t-1})}}
= \iprod{\bdelta_{t-1}}{\nabla \cL(\bw_{t-1})}
\geq \cL(\bw_{t-1}) + \frac{\alpha}{2}
\norm{\bdelta_{t-1}}^2
\end{equation}

By (\ref{eq:convex-1}) and (\ref{eq:convex-2}), we see that
\begin{equation}\label{eq:convex-3}
\expec{t_1, \ldots, t_m}{\norm{\bdelta_{t}}^2}
\leq
(1 - \eta \alpha) \norm{\bdelta_{t-1}}^2
- 2 \eta \cdot \expec{t_1, \ldots, t_m}
{\cL(\bw_{t-1}) - \frac{\eta}{2} \norm{\nabla L_{m,t}(\bw_{t-1})}^2}
\end{equation}

Next, we choose $\eta$ such that
$\expec{t_1, \ldots, t_m}
{\cL(\bw_{t-1}) - \frac{\eta}{2} \norm{\nabla L_{m,t}(\bw_{t-1})}^2} \geq 0$. We start by giving an important expansion,
\begin{equation}\label{eq:mb-norm}
\begin{split}
&\expec{t_1, \ldots, t_m}{\norm{\nabla L_{m,t}(\bw_{t-1})}^2}\\
&= 
\expec{t_1, \ldots, t_m}{\iprod{\frac{1}{m}
\sum_{i=1}^m \nabla \ell_{t_i}(\bw_{t-1})}{\frac{1}{m}
\sum_{i=1}^m \nabla \ell_{t_i}(\bw_{t-1})}}\\
&=
\frac{1}{m^2} \left\lbrace
\sum_{i=1}^m \expec{t_i}{\norm{\nabla \ell_{t_i} (\bw_{t-1})}^2}
+ \sum_{i=1}^m \sum_{j = 1 (j\neq i)}^m
\expec{t_i, t_j}{\iprod{\nabla \ell_{t_i} (\bw_{t-1})} {\nabla \ell_{t_j} (\bw_{t-1})}}
\right\rbrace\\
&=
\frac{1}{m} \expec{t_1}{\norm{\nabla L_{1,t}(\bw_{t-1})}^2} + \frac{m-1}{m} \norm{\nabla \cL(\bw_{t-1})}^2
\end{split}
\end{equation}

Then recalling the $\beta$-smoothness of $L_{1,t}(\bw) = \ell_{t_1}(\bw)$ and $\lambda$-smoothness of $\cL(\bw)$, we have
\begin{equation}\label{eq:smooth}
\begin{split}
&L_{1,t}(\bw) - \frac{1}{2\beta} \norm{\nabla L_{1,t}(\bw)}^2 \geq 0\\
&\cL(\bw) - \frac{1}{2 \lambda} 
\norm{\nabla \cL(\bw)}^2 \geq 0
\end{split}
\end{equation}

From (\ref{eq:mb-norm}) and (\ref{eq:smooth}), it is easy to see that for any $p \in [0, 1]$ when choosing $\eta(p) \defeq \min\{ \frac{p\cdot m}{\beta}, \frac{1 - p}{\lambda} \cdot \frac{m}{m-1} \}$, we have

\begin{equation*}\label{eq:convex-main}
\begin{split}
&\expec{t_1, \ldots, t_m}
{\cL(\bw_{t-1}) - \frac{\eta(p)}{2} \norm{\nabla L_{m,t}(\bw_{t-1})}^2}\\
& =
\expec{t_1}{p \cdot L_{1,t}(\bw_{t-1}) - \frac{\eta(p)}{2}\cdot \frac{1}{m} \norm{\nabla L_{1,t}(\bw_{t-1})}^2}
+ \{(1-p) \cdot \cL(\bw_{t-1}) - \frac{\eta(p)}{2}\cdot
\frac{m-1}{m}\norm{\nabla \cL(\bw_{t-1})}^2 \}\\
&\geq
p \cdot \expec{t_1}{ L_{1,t}(\bw_{t-1})
- \frac{1}{2\beta} \norm{\nabla L_{1,t}(\bw_{t-1})}^2}
+ (1-p) \{\cL(\bw_{t-1})
- \frac{1}{2 \lambda}
\norm{\nabla \cL(\bw_{t-1})}^2 \}
\geq 0
\end{split}
\end{equation*}

By dropping this term in (\ref{eq:convex-3}), we obtain
$$
\expec{t_1, \ldots, t_m}{\norm{\bdelta_{t}}^2}
\leq
(1 - \eta(p) \cdot \alpha) \norm{\bdelta_{t-1}}^2
$$

We see that for $p\in [0, 1]$, $1 - \eta(p) \cdot \alpha$ reaches its minimum (for fastest convergence) when $p = \frac{\beta}{\beta + \lambda (m - 1)}$.
Thus we choose $\eta^*(m) = \frac{m}{\beta + \lambda (m - 1)}$ corresponding to the best $p$ and obtain,
$$
\expec{t_1, \ldots, t_m}{\norm{\bdelta_{t}}^2}
\leq
(1 - \eta^*(m) \cdot \alpha) \norm{\bdelta_{t-1}}^2
$$

Incorporating this result with inequality~(\ref{eq:t1-smooth}), we have
\begin{equation*}
\begin{split}
\ex{\bw_t}{\cL(\bw_t)}
&\leq \frac{\lambda}{2} \ex{\bw_t}{\norm{\bw_t-\bw^*}^2}\\
&\leq \frac{\lambda}{2}
(1 - \eta^*(m) \cdot \alpha)
\ex{\bw_{t-1}}{\norm{\bw_{t-1}-\bw^*}^2}\\
&\leq \frac{\lambda}{2}
(1 - \eta^*(m) \cdot \alpha)^t
\norm{\bw_{0} - \bw^*}^2
\end{split}
\end{equation*}
\end{proof}

For $m=1$, this theorem is a special case of Theorem~2.1 in~\cite{needell2014stochastic}, which is a sharper version of Theorem~1 in~\cite{moulines2011non}.

\textbf{Speedup factor.} Let $t(m)$ be the number of iterations needed to reach a desired accuracy with batch size $m$.
Assuming $\lambda \gg \alpha$, the speed up factor $\frac{t(1)}{t(m)}$, which measures the number of iterations saved by using larger batch, is
$$
\frac{t(1)}{t(m)} 
= \frac{\log(1 - \eta^*(m)\alpha)}{\log(1 - \eta^*(1)\alpha)}
\approx \frac{\eta^*(m)}{\eta^*(1)}
= \frac{m\beta}{\beta + \lambda(m-1)}
$$

\textbf{Critical batch size}
$m^* \defeq \frac{\beta}{\lambda} + 1$ . By estimating the speedup factor for each batch size $m$, we directly obtain
\begin{itemize}
\item Linear scaling regime: one iteration of batch size $m \leq m^*$ is nearly equivalent to $m$ iterations of batch size $1$.
\item Saturation regime:
one iteration with batch size $m > m^*$ is nearly equivalent to one full gradient iteration.
\end{itemize}
We give a sharper analysis for the case of quadratic loss in Section~\ref{sec:fast}.

\subsection{Variance reduction methods in the interpolation regime}\label{sec:svrg}

For general convex optimization, a set of important stochastic methods~\cite{roux2012stochastic, johnson2013accelerating, defazio2014saga, xiao2014proximal,allen2016katyusha} have been proposed to achieve exponential (linear) convergence rate with constant step size.
The effectiveness of these methods derives from their ability to reduce the stochastic variance caused by sampling. In a general convex setting, this variance prevents SGD from both adopting a constant step size and achieving an exponential convergence rate.

\begin{wraptable}{r}{0.5\linewidth}
\centering
\begin{adjustbox}{center}
\resizebox{\linewidth}{!}{
\label{tbl:vr-conv}
\begin{tabular}{|c||c|c|l|}
\hline
\multirow{2}{*}{Method} & \multirow{2}{*}{Step size} & \multicolumn{2}{c|}{\multirow{2}{*}{\begin{tabular}[c]{@{}c@{}}\#Iterations to \\ reach a given error\end{tabular}}} \\
 &  & \multicolumn{2}{c|}{} \\ \hline
{\bf Mini-batch SGD}~(Theorem~\ref{thm:convex}) & 
$\frac{m}{\beta + \lambda (m - 1)}$ & \multicolumn{2}{c|}{$O(\frac{\beta + \lambda(m-1)}{m \alpha})$} \\ \hline
SGD (Eq.~\ref{eq:sgd}, m=1) & 
$\frac{1}{\beta}$ & \multicolumn{2}{c|}{$O(\frac{\beta}{\alpha})$} \\ \hline
SAG~\cite{roux2012stochastic} & $\frac{1}{2 n \cdot \beta}$ & \multicolumn{2}{c|}{$O(\frac{n \cdot \beta}{\alpha})$} \\ \hline
SVRG~\cite{johnson2013accelerating} & $\frac{1}{10 \beta}$ & \multicolumn{2}{c|}{$O(n + \frac{\beta}{\alpha})$} \\ \hline
SAGA~\cite{defazio2014saga} & $\frac{1}{3 \beta}$ & \multicolumn{2}{c|}{$O(n + \frac{\beta}{\alpha})$} \\ \hline
Katyusha~\cite{allen2016katyusha} (momentum) & adaptive & \multicolumn{2}{c|}{$O(n + \sqrt{\frac{n \cdot \beta}{\alpha}})$} \\ \hline
\end{tabular}
}
\end{adjustbox}
\vspace{-2mm}
\end{wraptable}
Remarkably, in the interpolated regime, Theorem~\ref{thm:convex} implies that SGD obtains the benefits of variance reduction ``for free" without the need for any modification or extra information (e.g., full gradient computations for variance reduction). 
The  table on the right compares the convergence of SGD in the interpolation setting with several popular variance reduction methods.
Overall, SGD has the largest step size and achieves the fastest convergence rate without the need for any further assumptions. The only comparable or faster rate is given by Katyusha, which is an accelerated SGD method combining momentum and variance reduction for faster convergence.

\section{How Fast is Fast SGD: Analysis of Step, Mini-batch Sizes and Computational Efficiency for Quadratic Loss}
\label{sec:fast}

In this section, 
we analyze the convergence of mini-batch SGD for quadratic losses. 
We will consider the following key questions:

\vspace{-1em}
\begin{itemize} 
\setlength{\itemsep}{-0.5em}
\item What is the optimal convergence rate of mini-batch SGD and the corresponding step size as a function of $m$ (size of mini-batch)?
\item What is the computational efficiency of different batch sizes and how do they compare to full GD?
\end{itemize}
\vspace{-1em}


The case of quadratic losses covers over-parametrized linear or kernel regression with a positive definite kernel. The quadratic case also captures general smooth convex functions in the neighborhood of a minimum where higher order terms can be ignored.

\paragraph{Quadratic loss.} Consider the problem of minimizing the sum of squares
\begin{align}
\cL(\bw) \triangleq \frac{1}{n} \sum_{i=1}^n (\bw^T \bx_i - y_i)^2 \nonumber
\end{align}
where $(\bx_i, y_i) \in \HH \times \mathbb{R}, i = 1, \ldots, n$ are labeled data points sampled from some (unknown) distribution.
In the interpolation setting, there exists $\bw^* \in \HH$ such that $L(\bw^*) = 0$. The covariance $H \triangleq \frac{1}{n} \sum_{i=1}^n {\bx_i \bx_i^T}$ can be expressed in terms of its eigen decomposition as $\sum_{i=1}^d \lambda_i \be_i \be_i^T$, where $d$ is the dimensionality of the parameter space (and the feature space) $\HH$, $\lambda_1\geq \lambda_2 \geq \cdots \geq\lambda_d$ are the eigenvalues of $H$, and $\{\be_1, \ldots, \be_d\}$ is the eigen-basis induced by $H$. In the over-parametrized setting (i.e., when $d>n$), the rank of $H$ is at most $n$. Assume, w.o.l.g., that the eigenvalues are such that $\lambda_1\geq\lambda_2\geq \cdots\geq\lambda_k>0 =\lambda_{k+1}=\cdots =\lambda_d$ for some $k\leq n$. We further assume that for all feature vectors \mbox{$\bx_i, i=1, \ldots, n$,} we have \mbox{$\norm{\bx_i}^2 \leq \beta$.} Note that this implies that the trace of $H$ is bounded from above by $\beta$, that is, \mbox{$\tr(H)\leq \beta$.} Thus, we have $\beta > \lambda_1 \geq \lambda_2 \geq \cdots \geq\lambda_k>0$. Hence, in the interpolation setting, we can write the sum of squares $\cL(\bw)$ as 
\begin{align}
\cL(\bw) &=(\bw-\bw^*)^T H (\bw- \bw^*)\label{eq:emp-quad-loss}
\end{align}

For any $\bv\in \HH$, let $\bP_\bv$ denote the projection of $\bv$ unto the subspace spanned by $\{\be_1, \ldots, \be_k\}$ and $\bQ_{\bv}$ denote the projection of $\bv$ unto the subspace spanned by $\{\be_{k+1}, \ldots, \be_{d}\}$. That is, $\bv=\bP_{\bv}+\bQ_{\bv}$ is the decomposition of $\bv$ into two orthogonal components: its projection onto $\mathsf{Range}(H)$ (i.e., the range space of $H$, which is the subspace spanned by $\{\be_1, \ldots, \be_k\}$) and its projection onto $\mathsf{Null}(H)$ (i.e., the null space of $H$, which is the subspace spanned by $\{\be_{k+1}, \ldots, \be_d\}$). Hence, the above quadratic loss can be written as

\begin{align}
\cL(\bw) &=\bP_{\bw-\bw^*}^T ~H ~\bP_{\bw- \bw^*} \label{eq:quad-loss-equiv}
\end{align}

To minimize the loss in this setting, consider the following SGD update with mini-batch of size $m$ and step size $\eta$:
\begin{equation}\label{eq:sgd}
\bw_{t+1} = \bw_t - \eta H_m (\bw_t - \bw^*)
\end{equation}
where $H_m \triangleq \frac{1}{m} \sum_{i=1}^m \tilde{\bx}_i \tilde{\bx}_i^T$ is a subsample covariance corresponding to a subsample of feature vectors $\{\tilde{\bx}_1, \ldots, \tilde{\bx}_m\}\subset \{\bx_1, \ldots, \bx_n\}$. 

Let $\bdelta_t\triangleq \bw_t-\bw^*$. Observe that we can write (\ref{eq:sgd}) as 

\begin{align}
\bP_{\bdelta_{t+1}}+ \bQ_{\bdelta_{t+1}}&= \bP_{\bdelta_t}+\bQ_{\bdelta_t} - \eta H_m \left(\bP_{\bdelta_t}+\bQ_{\bdelta_t}\right)\label{eq:sgd-decomp}
\end{align}

Now, we make the following simple claim (whose proof is given in the appendix).

\begin{claim}\label{claim:eigen-decomp}
Let $\bu\in \HH$. For any subsample $\{\tilde{\bx}_1,\ldots, \tilde{\bx}_m\}\subset \{\bx_1, \ldots, \bx_n\},$ let $H_m=\frac{1}{m} \sum_{i=1}^m \tilde{\bx}_i \tilde{\bx}_i^T$ be the corresponding subsample covariance matrix. Then, 
\begin{align}
H_m\bu &\in\mathsf{Range}(H)=\mathsf{Span}\{\be_1, \ldots, \be_k\}.\nonumber
\end{align}
This also implies that for any $\bv\in\mathsf{Null}(H)=\mathsf{Span}\{\be_{k+1}, \ldots, \be_{d}\},$ we must have $H_m\bv=0$.
\end{claim}

By the above claim, the update equation (\ref{eq:sgd-decomp}) can be decomposed into two components: 
 
\begin{align}
\bP_{\bdelta_{t+1}} &= \bP_{\bdelta_t} - \eta H_m \bP_{\bdelta_t},\label{eq:sgd-1}\\
\bQ_{\bdelta_{t+1}}&=\bQ_{\bdelta_t}\label{eq:sgd-2}
\end{align}

From (\ref{eq:quad-loss-equiv}), it follows that for any iteration $t$, the target loss function $\cL_{\bw_t}$ is not affected at all by $\bQ_{\bdelta_t}$, that is, $\bP_{\bdelta_t}$ is the only component that matters. Hence, by (\ref{eq:sgd-1}-\ref{eq:sgd-2}), we only need to consider the effective SGD update (\ref{eq:sgd-1}), i.e., the update in the span of $\{\be_1, \ldots, \be_k\}$.

\subsection{Upper bound on the expected empirical loss}
The following theorem provides an upper bound on the expected empirical loss after $t$ iterations of mini-batch SGD whose update step is given by (\ref{eq:sgd}). 

\begin{theorem}\label{thm:upper-bd}
For any $\lambda \in [\lambda_k, \lambda_1], m\in\mathbb{N},$ and $0<\eta<\frac{2m}{\beta + (m-1)\lambda_1 }\,$ define
$$g(\lambda; m, \eta) \triangleq (1 - \eta \lambda)^2 + \frac{\eta^2 \lambda}{m} (\beta - \lambda).$$
Let $g(m, \eta) \triangleq \max_{\lambda \in 
[\lambda_k, \lambda_1]}g(\lambda; m, \eta).$ In the interpolation setting, for any $t\geq 1,$ the mini-batch SGD with update step~(\ref{eq:sgd}) yields the following guarantee 
\begin{equation} 
\expect{\cL(\bw_{t})}\leq \lambda_1\cdot\expect{\norm{\bP_{\bdelta_{t}}^2}} \leq \lambda_1\cdot \left(g\left(m, \eta\right)\right)^t
\cdot \expect{\norm{\bP_{\bdelta_0}}^2}\nonumber
\end{equation}

\end{theorem}

\begin{proof}
By reordering terms in the update equation~\ref{eq:sgd-1} and using the independence of $H_m $ and $\bw_{t-1}$, the variance in the effective component of the parameter update can be written as
$$
\expect{\norm{\bP_{\bdelta_t}}^2}
= \expect{\bP_{\bdelta_{t-1}}^T {(I - 2 \eta H + \eta^2 \expect{H_m^2})} \bP_{\bdelta_{t-1}}}
$$
To obtain an upper bound, we need to bound $\expect{H_m^2}$. Notice that $H_m = \frac{1}{m}\sum_{i=1}^m H_1^{(i)}$ where $H_1^{(i)}, i = 1,\ldots, m$ are unit-rank independent subsample covariances. Expanding $H_m$ accordingly yield
\begin{align}
\expect{H_m^2}
= \frac{1}{m} \expect{H_1^2} + \frac{m-1}{m} H^2
\preceq \frac{\beta}{m} H + \frac{m-1}{m} H^2 \label{bound-Hm}
\end{align}

Let $G_{m, \eta} \triangleq I - 2 \eta H + \eta^2 (\frac{\beta}{m} H + \frac{m-1}{m} H^2)$. Then the variance is bounded as
$$
\expect{\norm{\bP_{\bdelta_t}}^2} \leq
\expect{\bP_{\bdelta_{t-1}}^T G_{m, \eta} \bP_{\bdelta_{t-1}}}
$$
Clearly, $\lim_{t\rightarrow \infty}{\expect{\norm{\bP_{\bdelta_t}}^2}} = 0$ if
\begin{equation}\label{eq:convergence-cond}
\norm{G_{m, \eta}} < 1
\Leftrightarrow
\eta < \eta_1(m)\triangleq \frac{2m}{\beta + (m-1)\lambda_1 }
\end{equation}
Furthermore, the convergence rate relies on the eigenvalues of $G_{m, \eta}$. Let $\lambda$ be a non-zero eigenvalue of $H$, then the corresponding eigenvalue of $G_{m, \eta}$ is given by
$$
g(\lambda; m, \eta) =
1 - 2\eta \lambda + \eta^2 [\frac{\beta}{m} \lambda + (1 - \frac{1}{m})\lambda^2]
= (1 - \eta \lambda)^2 + \frac{\eta^2 \lambda}{m} (\beta - \lambda)
$$
When the step size $\eta$ and mini-batch size $m$ are chosen (satisfying constraint~\ref{eq:convergence-cond}), we have
\begin{equation} 
\expect{\norm{\bP_{\bdelta_t}}^2} \leq g(m, \eta)
\cdot \expect{\norm{\bP_{\bdelta_{t-1}}^2}}\nonumber
\end{equation}
where 
\begin{equation} 
g(m, \eta) \triangleq \max_{\lambda \in \{ \lambda_1, \ldots, \lambda_k \}}g(\lambda; m, \eta).\nonumber
\end{equation}

Finally, observe that 
\begin{equation}
\cL(\bw_{t})\leq \lambda_1 \norm{\bP_{\bdelta_{t}}}^2\label{ineq:last-ineq-upper}
\end{equation}
which follows directly from (\ref{eq:quad-loss-equiv}).
\end{proof}

\subsection{Tightness of the bound on the expected empirical loss}
We now show that our upper bound given above is indeed tight in the interpolation setting for the class of quadratic loss functions defined in~(\ref{eq:emp-quad-loss}). Namely, we give a specific instance of (\ref{eq:emp-quad-loss}) where the upper bound in Theorem~\ref{thm:upper-bd} is tight.

\begin{theorem}~\label{thm:lower-bd}
There is a data set $\{(\bx_i, y_i)\in \HH\times \re: 1\leq i\leq n\}$ such that the mini-batch SGD with update step~(\ref{eq:sgd}) yields the following lower bound on the expected empirical quadratic loss $\cL(\bw)$
\begin{align}
\expect{\cL(\bw_{t})}&= \lambda_1\cdot\expect{\norm{\bdelta_t}^2} = \lambda_1\cdot \left(g\left(m, \eta\right)\right)^t
\cdot \expect{\norm{\bdelta_0}^2}\nonumber
\end{align}
\end{theorem}
\begin{proof}
We start the proof by observing that there are only two places in the proof of Theorem~\ref{thm:upper-bd} where the upper bound may not be tight, namely, the last inequality in~\ref{bound-Hm} and inequality (\ref{ineq:last-ineq-upper}).  Consider a data set where all the feature vectors $\bx_i, i=1, \ldots, n$, lie on the sphere of radius $\beta$, that is, $\norm{\bx_i}^2 = \beta,~\forall i=1, \ldots, n.$ We note that the last inequality in~\ref{bound-Hm} in the proof of Theorem~\ref{thm:upper-bd} is tight in that setting. Suppose that, additionally, we choose the feature vectors such that the eigenvalues of the sample covariance matrix $H$ are all equal, that is, $\lambda_1=\lambda_2=\cdots=\lambda_n$. This can be done, for example, by choosing all the feature vectors to be orthogonal (note that this is possible in the fully parametrized setting). Hence, in this case, (\ref{eq:emp-quad-loss}) implies
\begin{equation}
\cL(\bw_{t})= \lambda_1 \norm{\bdelta_{t}}^2\nonumber
\end{equation}
which shows that inequality (\ref{ineq:last-ineq-upper}) is also achieved with equality in that setting. This completes the proof.
\end{proof}

\noindent{\bf Remark.} From the experimental results, it appears that our upper bound can be close to tight even in some settings when the eigenvalues are far apart. We plan to investigate this phenomenon further.

\subsection{Optimal step size for a given batch size} \label{sec:optimal-eta}
To fully answer the first question we posed at the beginning of this section, we will derive an optimal rule for choosing the step size as a function of the batch size. Specifically, we want to find step size $\eta^*(m)$ to achieve fastest convergence. Given Theorem~\ref{thm:upper-bd}, our task reduces to finding the minimizer 
\begin{align}
\eta^*(m)=\arg\min_{\eta < \frac{2}{\frac{\beta}{m} + \frac{m-1}{m} \lambda_1 } } g(m, \eta)\label{eq:opt-step-def}
\end{align}

Let $g^*(m)$ denote the resulting minimum, that is, $g^*(m)=g\left(m, \eta^*(m)\right)$.
The resulting expression for the minimizer $\eta^*(m)$ generally depends on the least non-zero eigenvalue $\lambda_k$ of the Hessian matrix. In situations where we don't have a good estimate for this eigenvalue (which can be close to zero in practice), one would rather have a step size that is independent of $\lambda_k$. In Theorem \ref{thm:approx-step-size}, we give a near-optimal approximation for step size with no dependence on $\lambda_k$ under the assumption that $\beta/\lambda_k =\Omega(n)$, which is valid in many practical settings such as in kernel learning with positive definite kernels.

We first characterize exactly the optimal step size and the resulting $g^*(m)$. 

\begin{theorem}[Optimal step size role as function of batch size]\label{thm:opt-step-size}
For every batch size $m$, the optimal step size function $\eta^*(m)$ and convergence rate function $g^*(m)$ are given by:
\begin{align}
\eta^*(m)&=\left\{\begin{matrix}
 \frac{m}{\beta+(m-1)\lambda_k} &  m\leq \frac{\beta}{\lambda_1-\lambda_k}+1 \\ 
 \frac{2m}{\beta+(m-1)(\lambda_1+\lambda_k)} &  m> \frac{\beta}{\lambda_1-\lambda_k}+1
 \end{matrix}\right.\label{eq:opt-eta}\\
g^*(m)&=\left\{\begin{matrix}
 1 - \frac{m \lambda_k}{\beta + (m-1)\lambda_k} &  m\leq \frac{\beta}{\lambda_1-\lambda_k}+1 \\ 
 1 - 4\frac{m(m-1) \lambda_1\lambda_k}{\left(\beta + (m-1)(\lambda_1+\lambda_k)\right)^2} &  m> \frac{\beta}{\lambda_1-\lambda_k}+1
 \end{matrix}\right.\label{eq:opt-g}
 \end{align}

Note that if $\lambda_1=\lambda_k$, then the first case in each expression will be valid for all $m\geq 1$.
\end{theorem}
The proof of the above theorem follows from the following two lemmas. 

\begin{lemma}\label{lem:max-g}
Let $\eta_0(m)\triangleq \frac{2m}{\beta+(m-1)(\lambda_1+\lambda_k)}$, and let $\eta_1(m)\triangleq \frac{2m}{\beta + (m-1)\lambda_1 }$. Then,
\begin{align}
g(m, \eta)&=\left\{\begin{matrix}
 g^{\rom{1}}(m, \eta)\triangleq g(\lambda_k; m, \eta) &  \eta\leq \eta_{0}(m) \\ 
 g^{\rom{2}}(m, \eta)\triangleq g(\lambda_1; m, \eta) &  \eta_0(m)<\eta\leq\eta_{1}(m)
 \end{matrix}\right.\nonumber
\end{align}
\end{lemma}

\begin{proof}
For any fixed $m\geq 1$ and $\eta < \eta_1(m),$ observe that $g(\lambda; m, \eta)$ is a quadratic function of $\lambda$. Hence, the maximum must occur at either $\lambda=\lambda_k$ or $\lambda=\lambda_1$. Define $g^{\rom{1}}(m, \eta)\triangleq g(\lambda_k; m, \eta)$ and $g^{\rom{2}}(m, \eta)\triangleq g(\lambda_1; m, \eta)$. Now, depending on the value of $m$ and $\eta$, we would either have $g^{\rom{1}}(m, \eta)\geq g^{\rom{2}}(m, \eta)$ or $g^{\rom{1}}(m, \eta) < g^{\rom{2}}(m, \eta)$. In particular, it is not hard to show that 
$$g^{\rom{1}}(m, \eta)\geq g^{\rom{2}}(m, \eta)\Leftrightarrow \eta \leq \eta_0(m),$$
where $\eta_0(m)\triangleq \frac{2m}{\beta+(m-1)(\lambda_1+\lambda_k)}$. This completes the proof.
\end{proof}

\begin{lemma}\label{lem:opt-per-each-g}
Given the quantities defined in Lemma~\ref{lem:max-g}, let $\eta^{\rom{1}}(m)=\argmin\limits_{\eta\leq \eta_0(m)}g^{\rom{1}}(m, \eta)$, and $\eta^{\rom{2}}(m)=\argmin\limits_{\eta_0(m)<\eta\leq \eta_1(m)}g^{\rom{2}}(m, \eta)$. Then, we have 
\begin{enumerate}
\item For all $m\geq 1$, $g^{\rom{1}}\left(m, \eta^{\rom{1}}(m)\right)\leq g^{\rom{2}}\left(m, \eta^{\rom{2}}(m)\right)$. 
\item For all $m\geq 1$, $\eta^{\rom{1}}(m)=\eta^*(m)$ and $g^{\rom{1}}\left(m, \eta^{\rom{1}}(m)\right)= g^*(m)$, where $\eta^*(m)$ and $g^*(m)$ are as given by (\ref{eq:opt-eta}) and (\ref{eq:opt-g}), respectively, (in Theorem~\ref{thm:opt-step-size}).
\end{enumerate}
\end{lemma}

\begin{proof}
First, consider $g^{\rom{1}}(m, \eta)$. For any fixed $m$, it is not hard to show that the minimizer of $g^{\rom{1}}(m, \eta)$ as a function of $\eta$, constrained to $\eta\leq \eta_0(m)$, is given by $\min\left(\eta_0(m), \frac{m}{\beta + (m-1)\lambda_k}\right)\triangleq\eta^{\rom{1}}(m)$. That is, 
\begin{align*}
\eta^{\rom{1}}(m)&=\left\{\begin{matrix}
 \frac{m}{\beta+(m-1)\lambda_k} &  m\leq \frac{\beta}{\lambda_1-\lambda_k}+1 \\ 
 \eta_0(m)=\frac{2m}{\beta+(m-1)(\lambda_1+\lambda_k)} &  m> \frac{\beta}{\lambda_1-\lambda_k}+1
 \end{matrix}\right.
\end{align*}
Substituting $\eta=\eta^{\rom{1}}(m)$ in $g^{\rom{1}}(m, \eta)$, we get 
\begin{align*}
g^{\rom{1}}\left(m, \eta^{\rom{1}}(m)\right)&=\left\{\begin{matrix}
 1 - \frac{m \lambda_k}{\beta + (m-1)\lambda_k} &  m\leq \frac{\beta}{\lambda_1-\lambda_k}+1 \\ 
 1 - 4\frac{m(m-1) \lambda_1\lambda_k}{\left(\beta + (m-1)(\lambda_1+\lambda_k)\right)^2} &  m> \frac{\beta}{\lambda_1-\lambda_k}+1
 \end{matrix}\right.
\end{align*}
Note that $\eta^{\rom{1}}(m)$ and $g^{\rom{1}}\left(m, \eta^{\rom{1}}(m)\right)$ are equal to $\eta^*(m)$ and $g^*(m)$ given in Theorem~\ref{thm:opt-step-size}, respectively. This proves item~2 of the lemma. 

Next, consider $g^{\rom{2}}(m, \eta)$. Again, for any fixed $m$, one can easily show that the minimum of $g^{\rom{2}}(m, \eta)$ as a function of $\eta$, constrained to $\eta_0(m)<\eta\leq \eta_1(m)$, is actually achieved at the boundary $\eta=\eta_0(m)$. Hence, $\eta^{\rom{2}}(m)=\eta_0(m)$. Substituting this in $g^{\rom{2}}(m, \eta)$, we get 
\begin{align*}
g^{\rom{2}}\left(m, \eta^{\rom{2}}(m)\right)&=1 - 4\frac{m(m-1) \lambda_1\lambda_k}{\left(\beta + (m-1)(\lambda_1+\lambda_k)\right)^2},~~\forall m\geq 1.
\end{align*}
We conclude the proof by showing that for all $m\geq 1$, $g^{\rom{1}}\left(m, \eta^{\rom{1}}(m)\right)\leq g^{\rom{2}}\left(m, \eta^{\rom{2}}(m)\right).$ Note that for $m> \frac{\beta}{\lambda_1-\lambda_k}+1,$ $g^{\rom{1}}\left(m, \eta^{\rom{1}}(m)\right)$ and $g^{\rom{2}}\left(m, \eta^{\rom{2}}(m)\right)$ are identical. For $m\leq \frac{\beta}{\lambda_1-\lambda_k}+1,$ given the expressions above, one can verify that $g^{\rom{1}}\left(m, \eta^{\rom{1}}(m)\right)\leq g^{\rom{2}}\left(m, \eta^{\rom{1}}(m)\right)$. 

\end{proof}

\paragraph{Proof of Theorem~\ref{thm:opt-step-size}:} Given Lemma~\ref{lem:max-g} and item~1 of Lemma~\ref{lem:opt-per-each-g}, it follows that $\eta^{\rom{1}}(m)$ is the minimizer $\eta^*(m)$ given by (\ref{eq:opt-step-def}). Item~2 of Lemma~\ref{lem:opt-per-each-g} concludes the proof of the theorem. 

\noindent {\bf Nearly optimal step size with no dependence on $\lambda_k$:} In practice, it is usually easy to obtain a good estimate for $\lambda_1$, but it is hard to reliably estimate $\lambda_k$ which is typically much smaller than $\lambda_1$ (e.g.,~\cite{chaudhari2016entropy}). That is why one would want to avoid dependence on $\lambda_k$ in practical SGD algorithms. Under a mild assumption which is typically valid in practice, we can easily find an accurate approximation $\hat{\eta}(m)$ of optimal $\eta^*(m)$ that depends only on $\lambda_1$ and $\beta$.  Namely, we assume that $\lambda_k/\beta\leq 1/n$. In particular, this is always true in kernel learning with positive definite kernels, when the data points are distinct. 

The following theorem provides such approximation resulting in a nearly optimal convergence rate $\hat{g}(m)$.
\begin{theorem}\label{thm:approx-step-size}
Suppose that $\lambda_k/\beta\leq 1/n$. Let $\hat{\eta}(m)$ be defined as: 
\begin{align}
\hat{\eta}(m)&=\left\{\begin{matrix}
 \frac{m}{\beta\left(1+(m-1)/n\right)} &  m\leq \frac{\beta}{\lambda_1-\beta/n}+1 \\ 
 \frac{2m}{\beta+(m-1)(\lambda_1+\beta/n)} &  m> \frac{\beta}{\lambda_1-\beta/n}+1
 \end{matrix}\right.\label{eq:approx-eta}
\end{align}
Then, the step size $\hat{\eta}(m)$ yields the following upper bound on $g\left(m, \hat{\eta}(m)\right)$, denoted as $\hat{g}(m)$: 
\begin{align}
\hat{g}\left(m\right)&=\left\{\begin{matrix}
 1 - \frac{m \lambda_k}{\beta\left(1 + (m-1)/n\right)} &  m\leq \frac{\beta}{\lambda_1-\beta/n}+1 \\ 
 1 - 4\frac{m(m-1) \lambda_1\lambda_k}{\left(\beta + (m-1)(\lambda_1+\beta/n)\right)^2} &  m> \frac{\beta}{\lambda_1-\beta/n}+1
 \end{matrix}\right.\label{approx-g}
 \end{align}
\end{theorem}
\begin{proof}
The proof easily follows by observing that if $\lambda_k/\beta\leq 1/n$, then $\hat{\eta}(m)$ lies in the feasible region for the minimization problem in (\ref{eq:opt-step-def}). In particular, $\hat{\eta}(m)\leq \eta_0(m)$, where $\eta_0(m)$ is as defined in Lemma~\ref{lem:max-g}. The upper bound $\hat{g}\left(m\right)$ follows from substituting $\hat{\eta}(m)$ in $g^{\rom{1}}(m, \eta)$ defined in Lemma~\ref{lem:max-g}, then upper-bounding the resulting expression. 
\end{proof}

It is easy to see that the convergence rate $\hat{g}(m)$ resulting from the step size $\hat{\eta}$ is at most factor $1+O(m/n)$ slower than the optimal rate $g^*(m)$. This factor is negligible when $m \ll n$. 
Since we expect $n \gg \beta$, we can further approximate $\hat{\eta}(m)\approx m/\beta$ when $m\lessapprox \beta/\lambda_1$ and $\hat{\eta} \approx \frac{2m}{\beta +(m-1)\lambda_1}$ when $m \gtrapprox \beta/\lambda_1$.




\subsection{Batch size selection}
In this section, we will derive the optimal batch size given a fixed computational budget in terms of the computational efficiency defined as the number of gradient  computations to obtain a fixed desired accuracy.
We will show that single-point batch is in fact optimal in that setting. Moreover, we will show that any  mini-batch size in the range from $1$ to a certain constant $m^*$ independent of $n$, is nearly optimal in terms of gradient computations. 
Interestingly, for values beyond $m^*$ the computational efficiency drops sharply. This result has direct implications for the batch size selection in parallel computation. 
 


\subsubsection{Optimality of a single-point batch (standard SGD)}\label{sec:opt-single-pt}

Suppose we are limited by a fixed number of gradient computations. Then, what would be the batch size that yields the least approximation error? Equivalently, suppose we are required to achieve a certain target accuracy $\epsilon$ (i.e., want to reach parameter $\hat{\bw}$ such that $\cL(\hat{\bw})-\cL(\bw^*)\leq \epsilon$). Then, again, what would be the optimal batch size that yields the least amount of computation. 

Suppose we are being charged a unit cost for each gradient computation, then it is not hard to see that the cost function we seek to minimize is $g^*(m)^{\frac{1}{m}}$, where $g^*(m)$ is as given by Theorem~\ref{thm:opt-step-size}. To see this, note that for a batch size $m$, the number of iterations to reach a fixed desired accuracy is $t(m)=\frac{\mathsf{constant}}{\log(1/g^*(m))}$. Hence, the computation cost is $m\cdot t(m)=\frac{\mathsf{constant}}{\log(1/g^*(m))^{1/m}}$. Hence, minimizing the computation cost is tantamount to minimizing $g^*(m)^{1/m}.$ The following theorem shows that the exact minimizer is $m=1$. Later, we will see that \emph{any} value for $m$ from $2$ to $\approx \beta/\lambda_1$ is actually not far from optimal. So, if we have cheap or free computation available (e.g., parallel computation), then it would make sense to choose $m\approx \beta/\lambda_1$. We will provide more details in the following subsection. 

\begin{theorem}[Optimal batch size under a limited computational budget]\label{thm:opt-batch-size}
When we are charged a unit cost per gradient computation, the batch size that minimizes the overall computational cost required to achieve a fixed accuracy (i.e., maximizes the computational efficiency) is $m=1$. Namely,
$$
\arg\min_{m \in \mathbb{N}} g^*(m)^{\frac{1}{m}} = 1
$$
\end{theorem}

The detailed and precise proof is deferred to the appendix.
Here, we give a less formal but more intuitive argument based on a reasonable approximation for $g^*(m)$. Such approximation in fact is valid in most of the practical settings. In the full version of this paper, we give an exact and detailed analysis. 
\noindent Note that $g^*(m)$ can be written as $1-\frac{\lambda_k}{\beta}s(m)$, where $s(m)$ 
is given by
\begin{align}
s(m)&=\left\{\begin{matrix}
 \frac{m}{1 + (m-1)\frac{\lambda_k}{\beta}} &  m\leq \frac{\beta}{\lambda_1-\lambda_k}+1 \\ 
 \frac{4m(m-1) \lambda_1}{\beta\left(1 + (m-1)\frac{\lambda_1+\lambda_k}{\beta}\right)^2} &  m> \frac{\beta}{\lambda_1-\lambda_k}+1
 \end{matrix}\right.\label{eq:s(m)}
 \end{align}

\paragraph{Proof outline:} Note that $s(m)$ defined in (\ref{eq:s(m)}) indeed captures the speed-up factor we gain in convergence relative to standard SGD (with $m=1$) where the convergence is dictated by $\lambda_k/\beta$. Now, note that $g^*(m)^{1/m}\approx e^{-\lambda_k/ \beta \cdot s(m)/m}$. This approximation becomes very accurate when $\lambda_k \ll \lambda_1$, which is typically the case for most of the practical settings where $\lambda_1/\lambda_k\approx n$ and $n$ is very large. Assuming that this is the case (for the sake of this intuitive argument), minimizing $g^*(m)^{1/m}$ becomes equivalent to maximizing $s(m)/m$. Now, note that when $m\leq \frac{\beta}{\lambda_1-\lambda_k}+1,$ then $s(m)/m= \frac{1}{1 + (m-1)\frac{\lambda_k}{\beta}}$, which is decreasing in $m$. Hence, for $m\leq \frac{\beta}{\lambda_1-\lambda_k}+1,$ we have $s(m)/m\leq s(1)=1$. On the other hand, when $m > \frac{\beta}{\lambda_1-\lambda_k}+1,$ we have 
$$s(m)/m=\frac{4(m-1) \lambda_1}{\beta\left(1 + (m-1)\frac{\lambda_1+\lambda_k}{\beta}\right)^2},$$
which is also decreasing in $m$, and hence, it's upper bounded by its value at $m=m^*\triangleq \frac{\beta}{\lambda_1-\lambda_k}+1$. By direct substitution and simple cancellations, we can show that $s(m^*)/{m^*}\leq \frac{\lambda_1-\lambda_k}{\lambda_1}<1$. Thus, $m=1$ is optimal. 

One may wonder whether the above result is valid if the near-optimal step size $\hat{\eta}(m)$ (that does not depend on $\lambda_k$) is used. That is, one may ask whether the same optimality result is valid if the near optimal error rate function $\hat{g}(m)$ is used instead of $g^*(m)$ in Theorem~\ref{thm:opt-batch-size}. Indeed, we show that the same optimality remains true even if computational efficiency is measured with respect to $\hat{g}(m)$. This is formally stated in the following theorem. 

\begin{theorem}\label{thm:opt-batch-size-for-approx-g}
When the near-optimal step size $\hat{\eta}(m)$ is used (and assuming that $\lambda_k/\beta\leq 1/n)$, the batch size that minimizes the overall computational cost required to achieve a fixed accuracy is $m=1$. Namely,
$$
\arg\min_{m \in \mathbb{N}} \hat{g}(m)^{\frac{1}{m}} = 1
$$
\end{theorem}
The proof of the above theorem follows similar lines of the proof of Theorem~\ref{thm:opt-batch-size}.

\subsubsection{Near optimal larger batch sizes}

Suppose that several gradient computations can be performed in parallel. Sometimes doubling the number of machines used in parallel can halve the number of iterations needed to reach a fixed desired accuracy.
Such observation has motivated many works to use large batch size with distributed synchronized SGD~\cite{chen2016revisiting, goyal2017accurate, you2017scaling, smith2017don}.
One critical problem in this large batch setting is how to choose the step size. To keep the same covariance, \cite{bottou2016optimization, li2017scaling, hoffer2017train} choose the step size $\eta \sim \sqrt{m}$ for batch size $m$. While \cite{krizhevsky2014one, goyal2017accurate, you2017scaling, smith2017don} have observed that rescaling the step size $\eta \sim m$ works well in practice for not too large $m$.
To explain these observations, we directly connect the parallelism, or the batch size $m$,
to the required number of iterations $t(m)$ defined previously.
It turns out that (a) when the batch size is small, doubling the size will almost halve the required iterations; (b) after the batch size surpasses certain value, increasing the size to any amount would only reduce the required iterations by at most a constant factor.

\begin{figure}[H]
\centering
\includegraphics[width=.5\textwidth]{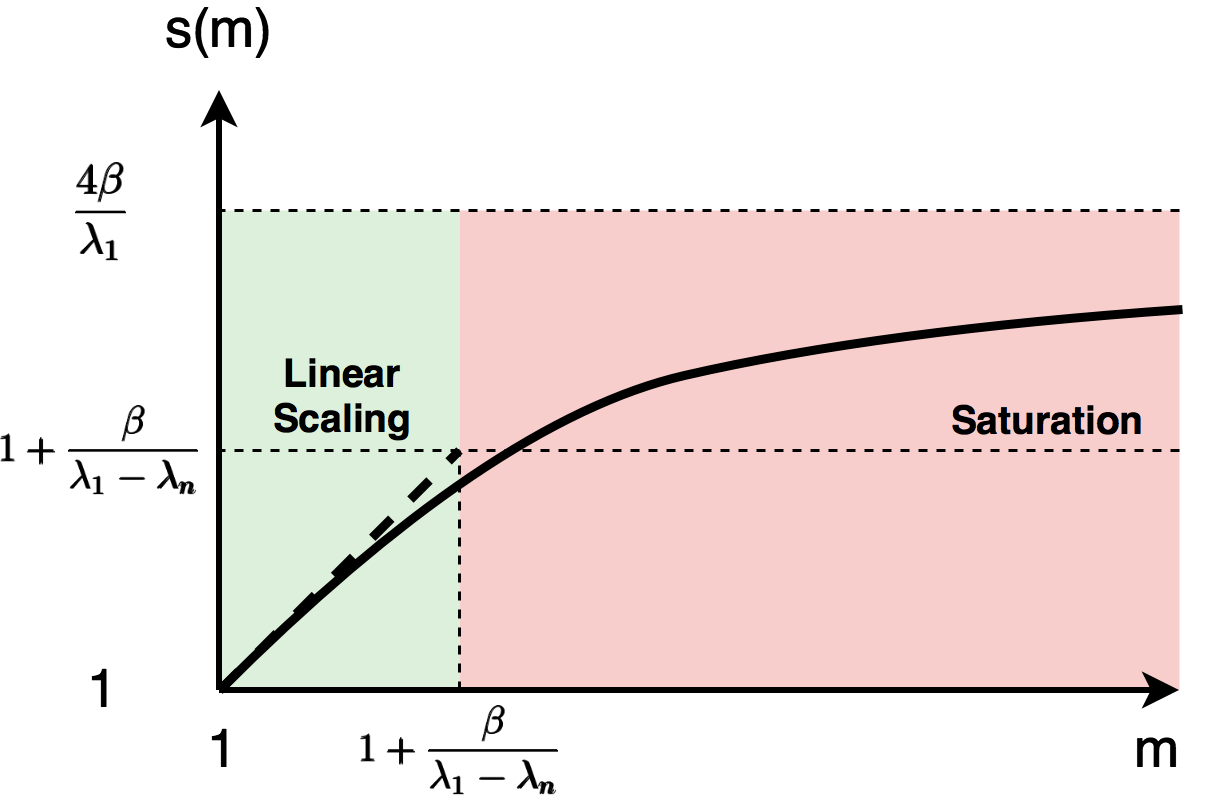}
\caption{Factor of iterations saved: $\frac{t(1)}{t(m)} \approx s(m)$}
\label{fig:optimal-region}
\end{figure}

Our analysis uses the optimal step size and convergence rate in Theorem~\ref{thm:opt-step-size}.
Now consider the factor by which we save the number of iterations when increasing the batch size from $1$ to $m$. Using the approximation $g^*(m)^{1/m} \approx e^{-\lambda_k/ \beta~\cdot~s(m)/m}$,
we have $\frac{t(1)}{t(m)} \approx s(m)$, the speed up factor. The change of $s(m)$ is illustrated in Figure~\ref{fig:optimal-region} where two regimes are highlighted:
\newline

\noindent {\bf Linear scaling regime} ($m \leq \frac{\beta}{\lambda_1-\lambda_k}+1$): This is the regime where increasing the batch size $m$ will quickly drive down $t(m)$ needed to reach certain accuracy. When $\lambda_k \ll \lambda_1$, $s(m) \approx m$, which suggests $t(m/2) \approx 2 \cdot t(m)$. In other words, doubling the batch size in this regime will roughly halve the number of iterations needed. 
Note that we choose step size $\eta \leftarrow \frac{m}{\beta + (m - 1)\lambda_k}$. When $\lambda_k\leq \frac{\beta}{n} \ll \lambda_1$, $\eta \sim m$, which is
consistent with the linear scaling heuristic used in \cite{krizhevsky2014one, goyal2017accurate, smith2017don}.
In this case, the largest batch size in the linear scaling regime can be practically calculated through
\begin{equation}\label{eq:m*}
m^* = \frac{\beta}{\lambda_1-\lambda_k}+1
\approx \frac{\beta}{\lambda_1- \beta/n}+1
\approx \frac{\beta}{\lambda_1} + 1
\end{equation}

\noindent {\bf Saturation regime} ($m > \frac{\beta}{\lambda_1-\lambda_k}+1$): Increasing batch size in this regime becomes much less beneficial.
Although $s(m)$ is monotonically increasing, it is upper bounded by $\lim_{m \rightarrow \infty}{s(m)} = \frac{4 \beta}{\lambda_1}$.
In fact, since $t(\frac{\beta}{\lambda_1-\lambda_k}+1) / \lim_{m \rightarrow \infty}{t(m)} < 4$ for small $\lambda_k$,
no batch size in this regime can reduce the needed iterations by a factor of more than 4.

\section{Experimental Results}
\label{sec:expr}

This section will provide empirical evidence for our theoretical results on the effectiveness of mini-batch SGD in the interpolated setting.  We first consider a kernel learning problem, where the parameters $\beta$, $\lambda_1$, and $m^*$ can be computed efficiently (see~\cite{ma2017diving} for details). In all experiments we set the step size  to be $\hat{\eta}$ defined in~(\ref{eq:approx-eta}).


\paragraph{Remark: near optimality of $\hat{\eta}$ in practice.} We observe empirically that increasing the step size from $\hat{\eta}$ to $2\, \hat{\eta}$ consistently leads to divergence, indicating that
$\hat{\eta}$ differs from the optimal step size by at most a factor of 2. This is consistent with
our Theorem~\ref{thm:approx-step-size} on near-optimal step size. 



\subsection{Comparison of SGD with critical mini-batch size $m^*$ to full gradient descent}
\begin{figure}[!ht]
  \centering
  \begin{minipage}[t]{0.3\textwidth}
    \includegraphics[width=\textwidth]
{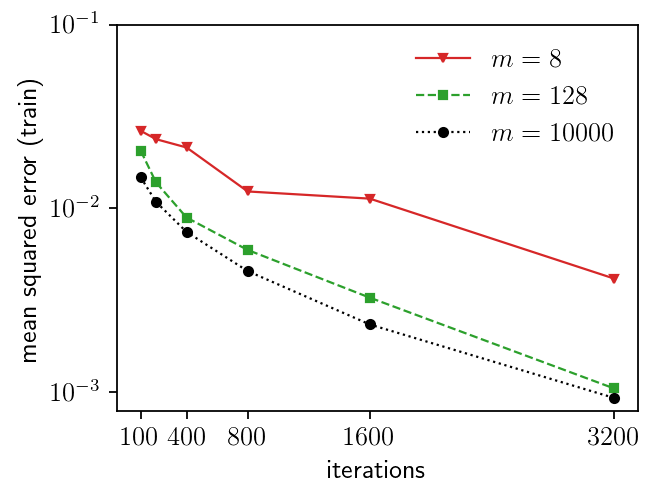}
    \subcaption{MNIST (Gaussian, $\sigma = 5$)\\
$\beta=1, \lambda_1=0.15, m^*\approx 8$}
  \end{minipage}
  \hfill
  \begin{minipage}[t]{0.3\textwidth}
    \includegraphics[width=\textwidth]
{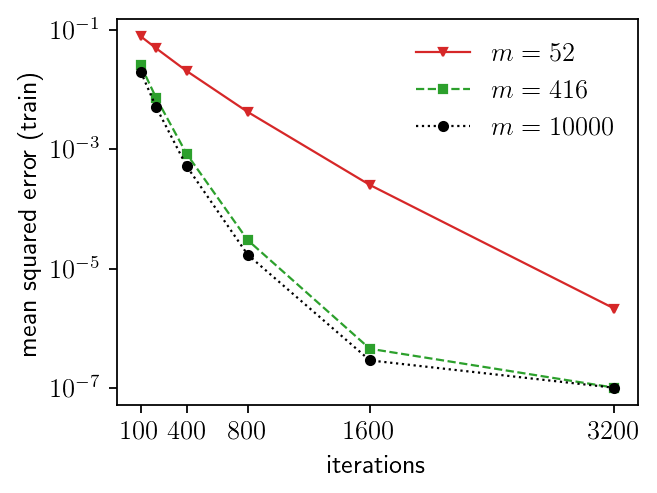}
    \subcaption{HINT-S (EigenPro-Laplace, $\sigma = 20$) $\beta=0.6, \lambda_1=0.012, m^*\approx 52$}
    \label{fig:mbs-iter-hint}
  \end{minipage}
  \hfill
  \begin{minipage}[t]{0.3\textwidth}
    \includegraphics[width=\textwidth]
{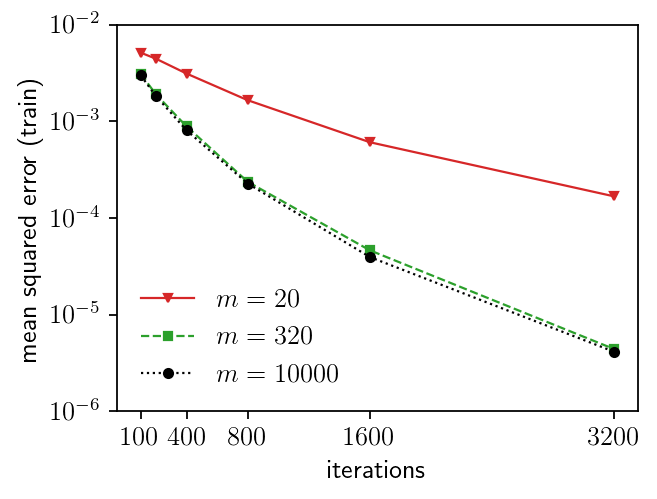}
    \subcaption{TIMIT (Gaussian, $\sigma=11$)\\
    $\beta=1, \lambda_1=0.054,$ $m^*\approx 20$}
  \end{minipage}
  \caption{Comparison of training error ($n = 10^4$) for different mini-batch sizes ($m$) vs. number of iterations}
  \label{fig:mbs-iter}
  \vspace{-3mm}
\end{figure}
Theorem~\ref{thm:opt-step-size} suggests that SGD using batch size $m^*$ defined in~(\ref{eq:m*}) can reach the same error as GD using at most $4$ times the number of iterations.
This is consistent with our experimental results for MNIST, HINT-S~\cite{healy2013algorithm}, and TIMIT, shown in Figure~\ref{fig:mbs-iter}.
Moreover, in line with our analysis, SGD with batch size larger than $m^*$ but still much smaller than the data size,   converges nearly identically to full gradient descent. 

\paragraph{Remark.} Since our analysis is concerned with the training error, only the training error is reported here. For completeness, we report the test error in  Appendix~\ref{sec:expr-test}. As consistently observed in such over-parametrized settings, test error  decreases with the training error.



\subsection{Optimality of batch size $m=1$}
Our theoretical results, Theorem~\ref{thm:opt-batch-size} and Theorem~\ref{thm:opt-batch-size-for-approx-g} show that 
$m=1$ achieves the optimal computational efficiency. Note for a given batch size, the corresponding optimal step size is chosen according to equation~(\ref{eq:approx-eta}). 
The experiments in Figure~\ref{fig:mbs-epoch} show that $m = 1$ indeed  achieves the lowest error for any fixed number of epochs. 

\subsection{Linear scaling and saturation regimes}
\begin{figure}[!t]
\centering 
  \begin{minipage}[t]{0.3\textwidth}
    \includegraphics[width=\textwidth]
{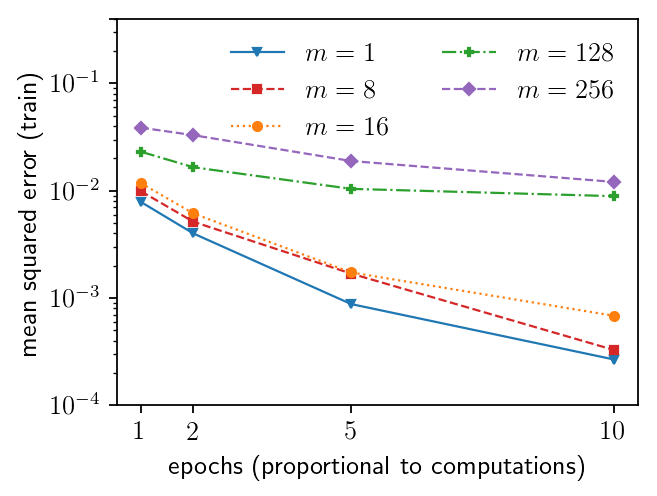}
    \subcaption{MNIST (Gaussian, $\sigma = 5$)\\
$\beta=1, \lambda_1=0.15, m^*\approx 8$}
  \end{minipage}
  \hfill
  \begin{minipage}[t]{0.3\textwidth}
    \includegraphics[width=\textwidth]
{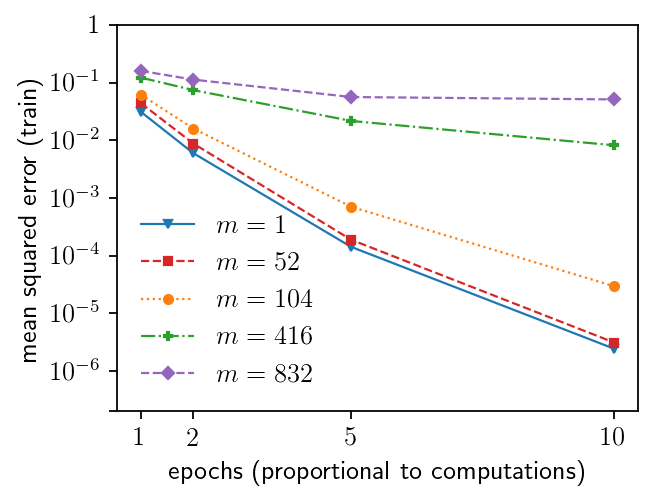}
    \subcaption{HINT-S (EigenPro-Laplace, $\sigma = 20$) $\beta=0.6, \lambda_1=0.012, m^*\approx 52$}  \end{minipage}
  \hfill
  \begin{minipage}[t]{0.3\textwidth}
    \includegraphics[width=\textwidth]
{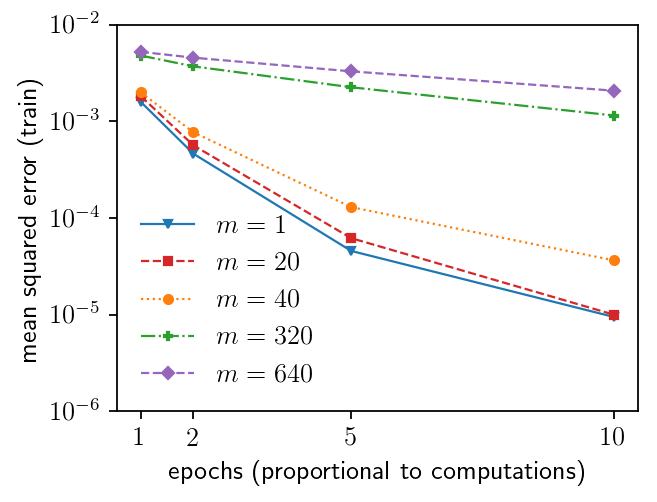}
    \subcaption{TIMIT (Gaussian, $\sigma=11$)\\
    $\beta=1, \lambda_1=0.054,$ $m^*\approx 20$}
  \end{minipage}
  \caption{Comparison of training error ($n = 10^4$) for different mini-batch sizes ($m$) vs. number of epochs (proportional to computation, note for $n$ data points, $n \cdot N_{epoch}= m \cdot N_{iter} $)}
  \label{fig:mbs-epoch}
  \vspace{-5mm}
\end{figure}
In the interpolation regime, Theorem~\ref{thm:opt-batch-size} shows linear scaling for mini-batch sizes up to a (typically small) ``critical'' batch size $m^*$ defined in~(\ref{eq:m*}) followed by the saturation regime.
In Figure~\ref{fig:mbs-epoch} we plot the training error for different batch sizes as a function of the number of epochs. Note that the number of epochs is proportional to the amount of computation measured in terms of gradient evaluations.
The linear scaling regime ($1 \leq m \leq m^*$) is reflected in the small difference in the training error for $m=1$ and $m=m^*$ in Figure~\ref{fig:mbs-epoch} (the bottom three curves. As expected from our theoretical results, they have similar computational efficiency. 
On the other hand, we see that large mini-batch sizes ($m\gg m^*$) require drastically more computations, which is the saturation phenomenon reflected in the top two curves.\\

\noindent{\bf Relation to the ``linear scaling rule'' in neural networks. }
A number of recent large scale neural network methods including~\cite{krizhevsky2014one,chen2016revisiting,goyal2017accurate} use the
``linear scaling rule'' to accelerate training using parallel computation. After the initial ``warmup'' stage to find a good region of parameters, this rule suggest increasing the step size to a level proportional to the mini-batch size $m$. In spite of the wide adoption and effectiveness of this technique, there has been no satisfactory explanation~\cite{krizhevsky2014one} as the usual variance-based analysis suggests increasing the step size by a factor of $\sqrt{m}$ instead of $m$~\cite{bottou2016optimization}. We note that this ``linear scaling'' can be explained by our analysis, assuming that the warmup stage ends up in a neighborhood of an interpolating minimum.



\subsection{Interpolation in kernel methods}
\begin{figure}[!ht]
  \centering
  \begin{minipage}[b]{0.55\textwidth}
    \includegraphics[width=\textwidth]{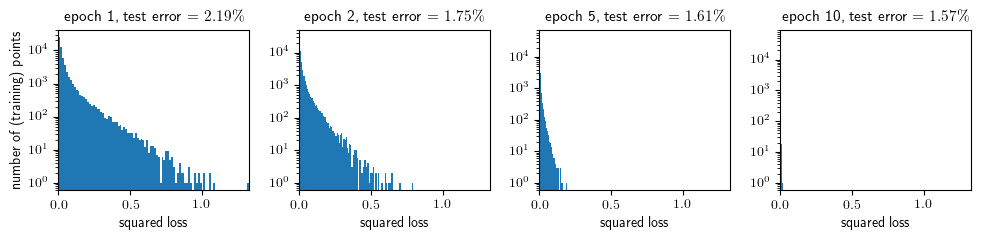}
    \subcaption{MNIST ($\sigma = 10$), $10$ classes}
  \end{minipage}
  \hfill
  \begin{minipage}[b]{0.4125\textwidth}
    \includegraphics[width=\textwidth]{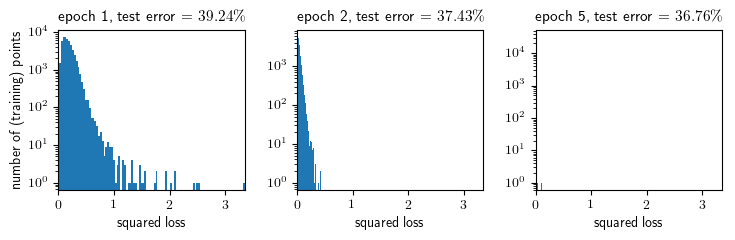}
    \subcaption{TIMIT ($\sigma = 20$), $144$ classes}
  \end{minipage}
  \caption{Histogram of training loss on data points at each epoch}
  \label{fig:interpolation}
\end{figure}
To give additional evidence of the interpolation  in the over-parametrized settings, we provide empirical results showing that this is indeed the case in kernel learning. 
We give two examples: Laplace kernel trained using EigenPro~\cite{ma2017diving} on MNIST~\cite{lecun1998gradient} and on a subset (of size $5 \cdot 10^4$) of TIMIT~\cite{garofolo1993darpa}.
The histograms in Figure~\ref{fig:interpolation} show the number of points with a given loss calculated as $\norm{\by_i - f(\bx_i)}^2$ (on feature vector $\bx_i$ and corresponding binary label vector $\by_i$).  As evident from the histograms, the test loss keeps decreasing as we  converge to an interpolated solution.

\paragraph{Relative computational efficiency is consistent with theory regardless of the rate of convergence.}
\begin{wrapfigure}{r}{0.5\textwidth}
  \vspace{-4mm}
  \centering
  \begin{minipage}[b]{0.5\textwidth}
    \includegraphics[width=\textwidth]{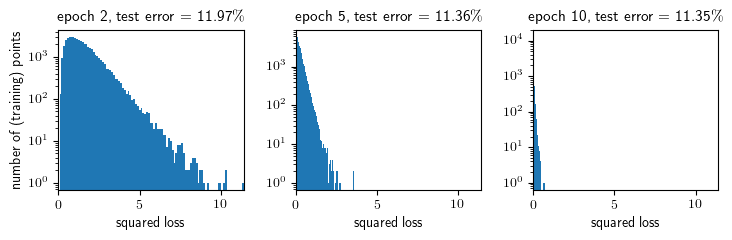}
    \vspace{-7mm}
    \subcaption{EigenPro-Laplace kernel, $\sigma=20$}
    \label{fig:hint-int-elap}
  \end{minipage}
  \begin{minipage}[b]{0.5\textwidth}
    \includegraphics[width=\textwidth]{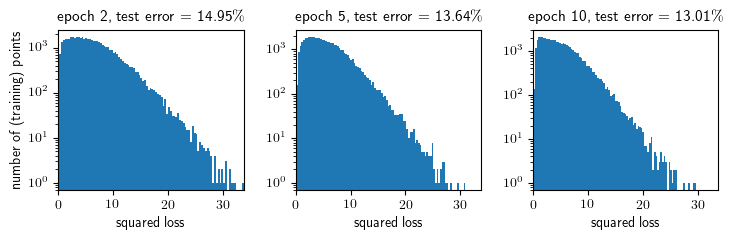}
    \vspace{-7mm}
    \subcaption{Gaussian kernel, $\sigma = 16$}
    \label{fig:hint-int-gaus}
  \end{minipage}
  \vspace{-7mm}
  \caption{Histogram of training loss on ($5 \cdot 10^4$) subsamples of HINT-S}
  \label{fig:hint-int}
  \vspace{-3mm}
\end{wrapfigure}
It is interesting to observe that even when  SGD is slow to converge to the interpolated
\begin{wrapfigure}{c}{\textwidth}
  \vspace{-2mm}
  \centering
  \captionsetup[subfigure]{justification=centering}
  \begin{minipage}[t]{0.35\textwidth}
    \includegraphics[width=\textwidth]
{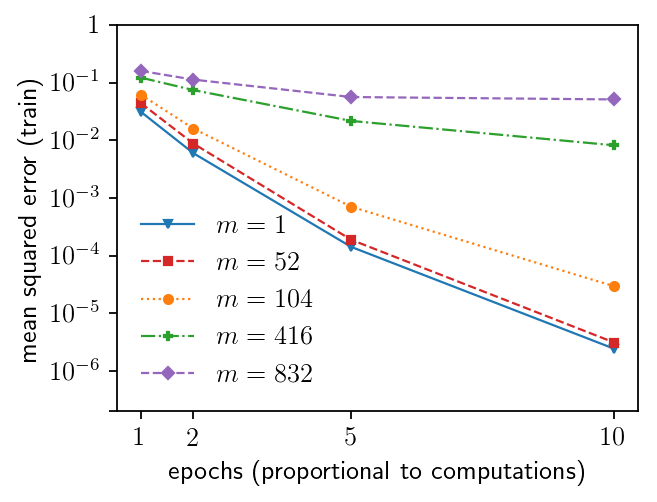}
\vspace{-7mm}
    \subcaption{
    EigenPro-Laplace, $m^* \approx 52$}
    \label{fig:hint-gaus-iter}
  \end{minipage}
  \begin{minipage}[t]{0.35\textwidth}
    \includegraphics[width=\textwidth]
{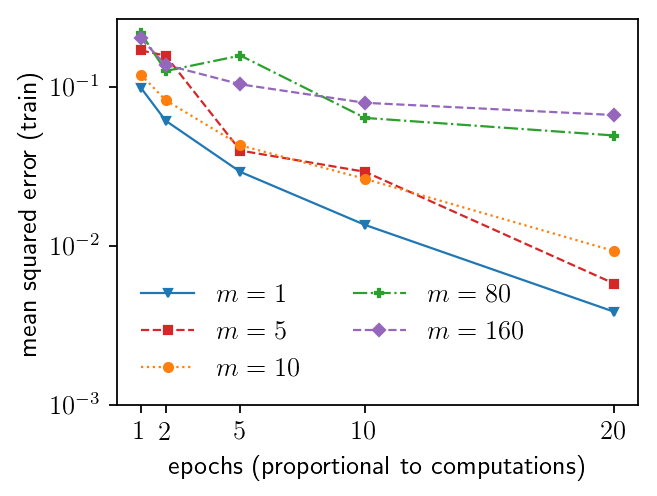}
    \vspace{-7mm}
    \subcaption{Gaussian, $m^* \approx 4$}
    \label{fig:hint-gaus-epoch}
  \end{minipage}
  \vspace{-3mm}
  \caption{Comparison of training error ($n = 10^4$) on HINT-S ($\sigma=20$) using different kernels}
  \label{fig:hint-gaus}
  \vspace{-3mm}
\end{wrapfigure}
solution, our theoretical bounds  still accurately describe the relative efficiency of different mini-batch sizes. 
We examine two different settings: interpolation with Laplace kernel trained using EigenPro~\cite{ma2017diving} on HINT-S (Figure~\ref{fig:hint-int-elap}) and with Gaussian kernel in Figure~\ref{fig:hint-int-gaus}. As clear from the figures Laplace kernel converges to the interpolated solution much faster than the Gaussian. 
However, as our experiments depicted in Figure~\ref{fig:hint-gaus} show, relative computational efficiency of different batch sizes for these two settings is very similar.
As before, we plot the training error against the number of epochs (which is proportional to computation) for different batch sizes. 
Note that while the scale of the error is very different for these two settings, the profiles of the curves are remarkably similar.


\clearpage
\bibliographystyle{alpha}
\bibliography{ref}

\newcommand{\etalchar}[1]{$^{#1}$}
\begin{thebibliography}{HLWvdM16}

\bibitem[AZ16]{allen2016katyusha}
Zeyuan Allen-Zhu.
\newblock Katyusha: The first direct acceleration of stochastic gradient
  methods.
\newblock {\em arXiv preprint arXiv:1603.05953}, 2016.

\bibitem[BCN16]{bottou2016optimization}
L{\'e}on Bottou, Frank~E Curtis, and Jorge Nocedal.
\newblock Optimization methods for large-scale machine learning.
\newblock {\em arXiv preprint arXiv:1606.04838}, 2016.

\bibitem[BFT17]{bartlett2017spectrally}
Peter Bartlett, Dylan~J Foster, and Matus Telgarsky.
\newblock Spectrally-normalized margin bounds for neural networks.
\newblock In {\em NIPS}, 2017.

\bibitem[CCSL16]{chaudhari2016entropy}
Pratik Chaudhari, Anna Choromanska, Stefano Soatto, and Yann LeCun.
\newblock Entropy-sgd: Biasing gradient descent into wide valleys.
\newblock {\em arXiv preprint arXiv:1611.01838}, 2016.

\bibitem[CMBJ16]{chen2016revisiting}
Jianmin Chen, Rajat Monga, Samy Bengio, and Rafal Jozefowicz.
\newblock Revisiting distributed synchronous sgd.
\newblock {\em arXiv preprint arXiv:1604.00981}, 2016.

\bibitem[CPC16]{canziani2016analysis}
Alfredo Canziani, Adam Paszke, and Eugenio Culurciello.
\newblock An analysis of deep neural network models for practical applications.
\newblock {\em arXiv preprint arXiv:1605.07678}, 2016.

\bibitem[DBLJ14]{defazio2014saga}
Aaron Defazio, Francis Bach, and Simon Lacoste-Julien.
\newblock Saga: A fast incremental gradient method with support for
  non-strongly convex composite objectives.
\newblock In {\em NIPS}, 2014.

\bibitem[GAGN15]{gupta2015deep}
Suyog Gupta, Ankur Agrawal, Kailash Gopalakrishnan, and Pritish Narayanan.
\newblock Deep learning with limited numerical precision.
\newblock In {\em ICML}, 2015.

\bibitem[GDG{\etalchar{+}}17]{goyal2017accurate}
Priya Goyal, Piotr Doll{\'a}r, Ross Girshick, Pieter Noordhuis, Lukasz
  Wesolowski, Aapo Kyrola, Andrew Tulloch, Yangqing Jia, and Kaiming He.
\newblock Accurate, large minibatch sgd: Training imagenet in 1 hour.
\newblock {\em arXiv preprint arXiv:1706.02677}, 2017.

\bibitem[GLF{\etalchar{+}}93]{garofolo1993darpa}
John~S Garofolo, Lori~F Lamel, William~M Fisher, Jonathon~G Fiscus, and David~S
  Pallett.
\newblock Darpa timit acoustic-phonetic continous speech corpus cd-rom.
\newblock {\em NIST speech disc}, 1-1.1, 1993.

\bibitem[HHS17]{hoffer2017train}
Elad Hoffer, Itay Hubara, and Daniel Soudry.
\newblock Train longer, generalize better: closing the generalization gap in
  large batch training of neural networks.
\newblock {\em arXiv preprint arXiv:1705.08741}, 2017.

\bibitem[HLWvdM16]{huang2016densely}
Gao Huang, Zhuang Liu, Kilian~Q Weinberger, and Laurens van~der Maaten.
\newblock Densely connected convolutional networks.
\newblock {\em arXiv preprint arXiv:1608.06993}, 2016.

\bibitem[HYWW13]{healy2013algorithm}
Eric~W Healy, Sarah~E Yoho, Yuxuan Wang, and DeLiang Wang.
\newblock An algorithm to improve speech recognition in noise for
  hearing-impaired listeners.
\newblock {\em The Journal of the Acoustical Society of America}, 134(4), 2013.

\bibitem[JKK{\etalchar{+}}16]{jain2016parallelizing}
Prateek Jain, Sham~M Kakade, Rahul Kidambi, Praneeth Netrapalli, and Aaron
  Sidford.
\newblock Parallelizing stochastic approximation through mini-batching and
  tail-averaging.
\newblock {\em arXiv preprint arXiv:1610.03774}, 2016.

\bibitem[JZ13]{johnson2013accelerating}
Rie Johnson and Tong Zhang.
\newblock Accelerating stochastic gradient descent using predictive variance
  reduction.
\newblock In {\em NIPS}, 2013.

\bibitem[Kac37]{kaczmarz1937angenaherte}
Stefan Kaczmarz.
\newblock Angenaherte auflosung von systemen linearer gleichungen.
\newblock {\em Bull. Int. Acad. Sci. Pologne, A}, 35, 1937.

\bibitem[Kri14]{krizhevsky2014one}
Alex Krizhevsky.
\newblock One weird trick for parallelizing convolutional neural networks.
\newblock {\em arXiv preprint arXiv:1404.5997}, 2014.

\bibitem[LBBH98]{lecun1998gradient}
Y.~LeCun, L.~Bottou, Y.~Bengio, and P.~Haffner.
\newblock Gradient-based learning applied to document recognition.
\newblock In {\em Proceedings of the IEEE}, volume~86, 1998.

\bibitem[Li17]{li2017scaling}
Mu~Li.
\newblock {\em Scaling Distributed Machine Learning with System and Algorithm
  Co-design}.
\newblock PhD thesis, 2017.

\bibitem[LZCS14]{li2014efficient}
Mu~Li, Tong Zhang, Yuqiang Chen, and Alexander~J Smola.
\newblock Efficient mini-batch training for stochastic optimization.
\newblock In {\em KDD}, 2014.

\bibitem[MB11]{moulines2011non}
Eric Moulines and Francis~R Bach.
\newblock Non-asymptotic analysis of stochastic approximation algorithms for
  machine learning.
\newblock In {\em NIPS}, 2011.

\bibitem[MB17]{ma2017diving}
Siyuan Ma and Mikhail Belkin.
\newblock Diving into the shallows: a computational perspective on large-scale
  shallow learning.
\newblock {\em arXiv preprint arXiv:1703.10622}, 2017.

\bibitem[NWS14]{needell2014stochastic}
Deanna Needell, Rachel Ward, and Nati Srebro.
\newblock Stochastic gradient descent, weighted sampling, and the randomized
  kaczmarz algorithm.
\newblock In {\em NIPS}, 2014.

\bibitem[RSB12]{roux2012stochastic}
Nicolas~L Roux, Mark Schmidt, and Francis~R Bach.
\newblock A stochastic gradient method with an exponential convergence rate for
  finite training sets.
\newblock In {\em NIPS}, 2012.

\bibitem[Sal17]{russ17}
Ruslan Salakhutdinov.
\newblock Deep learning tutorial at the {S}imons {I}nstitute, {B}erkeley,
  https://simons.berkeley.edu/talks/ruslan-salakhutdinov-01-26-2017-1, 2017.

\bibitem[SEG{\etalchar{+}}17]{sagun2017empirical}
Levent Sagun, Utku Evci, V~Ugur Guney, Yann Dauphin, and Leon Bottou.
\newblock Empirical analysis of the hessian of over-parametrized neural
  networks.
\newblock {\em arXiv preprint arXiv:1706.04454}, 2017.

\bibitem[SFBL98]{schapire1998}
Robert~E. Schapire, Yoav Freund, Peter Bartlett, and Wee~Sun Lee.
\newblock Boosting the margin: a new explanation for the effectiveness of
  voting methods.
\newblock {\em Ann. Statist.}, 26(5), 1998.

\bibitem[SKL17]{smith2017don}
Samuel~L Smith, Pieter-Jan Kindermans, and Quoc~V Le.
\newblock Don't decay the learning rate, increase the batch size.
\newblock {\em arXiv preprint arXiv:1711.00489}, 2017.

\bibitem[SV09]{strohmer2009randomized}
Thomas Strohmer and Roman Vershynin.
\newblock A randomized kaczmarz algorithm with exponential convergence.
\newblock {\em Journal of Fourier Analysis and Applications}, 15(2), 2009.

\bibitem[TBRS13]{takac2013mini}
Martin Tak{\'a}c, Avleen~Singh Bijral, Peter Richt{\'a}rik, and Nati Srebro.
\newblock Mini-batch primal and dual methods for svms.
\newblock In {\em ICML}, 2013.

\bibitem[XZ14]{xiao2014proximal}
Lin Xiao and Tong Zhang.
\newblock A proximal stochastic gradient method with progressive variance
  reduction.
\newblock {\em SIAM Journal on Optimization}, 24(4), 2014.

\bibitem[YGG17]{you2017scaling}
Yang You, Igor Gitman, and Boris Ginsburg.
\newblock Scaling sgd batch size to 32k for imagenet training.
\newblock {\em arXiv preprint arXiv:1708.03888}, 2017.

\bibitem[YPL{\etalchar{+}}18]{yin2018gradient}
Dong Yin, Ashwin Pananjady, Max Lam, Dimitris Papailiopoulos, Kannan
  Ramchandran, and Peter Bartlett.
\newblock Gradient diversity: a key ingredient for scalable distributed
  learning.
\newblock In {\em AISTATS}, 2018.

\bibitem[ZBH{\etalchar{+}}16]{zhang2016understanding}
Chiyuan Zhang, Samy Bengio, Moritz Hardt, Benjamin Recht, and Oriol Vinyals.
\newblock Understanding deep learning requires rethinking generalization.
\newblock {\em arXiv preprint arXiv:1611.03530}, 2016.

\end{thebibliography}

\clearpage
\appendix

\section{Proof of Claim~\ref{claim:eigen-decomp}}
Let $\{\be_1, \ldots, \be_d\}$ denote the eigen-basis of $H$ corresponding to eigenvalues $\lambda_1\geq\cdots\geq\lambda_k>0=\lambda_{k+1}=\cdots =\lambda_d$. For every $i\in \{1, \ldots, n\},$ let $\bx_i=\sum_{j=1}^d \alpha_{i,j}\be_j$ be the expansion of $\bx_i$ w.r.t. the eigen-basis of $H$. 

Observe that for any $k+1\leq\ell\leq d,$
\begin{align}
0=\lambda_{\ell}&=\be_{\ell}^T H\be_{\ell}=\frac{1}{n}\sum_{i=1}^n\alpha_{i,\ell}^2\nonumber
\end{align}
where the last equality follows from expanding each $\bx_i$ w.r.t. the eigen-basis of $H$. Thus, 
\begin{align}
\forall~ \ell\in \{k+1, \ldots, d\}, \alpha_{i, \ell}=0~~ \forall i\in \{1, \ldots, n\}.\label{eq:expansion-1} 
\end{align}

Fix any $\be_{r}\in \{\be_1, \ldots, \be_d\}$. Fix a collection $\{\tilde{\bx}_1,\ldots, \tilde{\bx}_m\}\subset \{\bx_1, \ldots, \bx_n\}$. Let $H_m = \frac{1}{m} \sum_{i=1}^m \tilde{\bx}_i \tilde{\bx}_i^T$. Now, from (\ref{eq:expansion-1}), we have 
\begin{align}
H_m\be_{r}&=\left\{\begin{matrix}
 \frac{1}{m}\sum_{i=1}^m \sum_{j=1}^k\tilde{\alpha}_{i,j}\tilde{\alpha}_{i,r}\be_j & 1\leq r\leq k\\ 
 0 & k+1\leq r\leq d
 \end{matrix}\right.\label{eq:expansion-2}
\end{align}
where $\tilde{\alpha}_{i, j}$ denotes the $j$-th coefficient of the expansion of $\tilde{\bx}_i$ w.r.t. the eigen-basis of $H$. 

The proof immediately follows from (\ref{eq:expansion-2}) since for any $\bu\in\HH$, we can write $\bu=\bP_{\bu}+\bQ_{\bu}$ where $\bP_{\bu}$ and $\bQ_{\bu}$ denote the projections of $\bu$ onto $\mathsf{Span}\{\be_1, \ldots, \be_k\}$ and $\mathsf{Span}\{\be_{k+1}, \ldots, \be_d\}$, respectively. Hence, (\ref{eq:expansion-2}) implies that $H_m \bP_{\bu}\in\mathsf{Span}\{\be_1, \ldots, \be_k\}$ and $H_m \bQ_{\bu}=0$, which proves the claim.

\section{Proof of Theorem \ref{thm:opt-batch-size}}
Here, we will provide an exact analysis for the optimality of batch size $m=1$ for the cost function $g^*(m)^{1/m}$, which, as discussed in Section~\ref{sec:opt-single-pt}, captures the total computational cost required to achieve any fixed target accuracy (in a model with no parallel computation). 

We prove this theorem by showing that $g^*(m)^{\frac{1}{m}}$ is strictly increasing for $m\geq 1$. We do this via the following two simple lemmas. First, we introduce the following notation. 

Let 
$$g_1(m)=1 - \frac{m \lambda_k}{\beta + (m-1)\lambda_k} , ~m\geq 1$$
That is, $g_1(m)$ is an extension of $g^*(m), ~m \in [1, \frac{\beta}{\lambda_1-\lambda_k}+1]$ (given by the first expression in (\ref{eq:opt-g}) in Theorem~\ref{thm:opt-step-size}) to all $m\geq 1.$

Let $g_2(m)$ denote the extension of $g^*(m), ~m > \frac{\beta}{\lambda_1-\lambda_k}+1$ (given by the second expression in (\ref{eq:opt-g}) in Theorem~\ref{thm:opt-step-size}) to all $m\geq 1$. That is, 
$$g_2(m)=1 - 4\frac{m(m-1) \lambda_1\lambda_k}{\left(\beta + (m-1)(\lambda_1+\lambda_k)\right)^2}, ~ m\geq 1.$$

\begin{lemma}\label{lem:proof-opt-batch1}
$g_1(m)^{\frac{1}{m}}$ is strictly increasing for $m\geq 1$. 
\end{lemma}
\begin{proof}
Define $T(m)\triangleq \frac{1}{m}\ln(1/ g_1(m))$. We will show that $T(m)$ is strictly decreasing for $m\geq 1$, which is tantamount to showing that $g_1(m)^{\frac{1}{m}}$ is strictly increasing over $m\geq 1$. For more compact notation, let's define $\tau\triangleq\frac{\beta-\lambda_k}{\beta}$, and $\bar{\tau}=1-\tau$. First note that, after straightforward simplification, $g^*(m)=\frac{\tau}{\tau+\bar{\tau}m}$. Hence, $T(m)=\frac{1}{m}\ln(1+u m),$ where $u\triangleq \frac{\bar{\tau}}{\tau}=\frac{\lambda_k}{\beta-\lambda_k}>0$.  Now, it is not hard to see that $T(m)$ is strictly decreasing since the function $\frac{1}{x}\ln(1+ u x)$ is strictly decreasing in $x$ as long as $u>0$. 
\end{proof}

\begin{lemma}\label{lem:proof-opt-batch2}
$g_1(m)\leq g_2(m)$, for all $m\geq 1.$
\end{lemma}

\begin{proof}
Proving the lemma is equivalent to proving $\frac{4m(m-1)\lambda_1\lambda_k}{(\beta+(m-1)(\lambda_1+\lambda_k)^2)}<\frac{m\lambda_k}{\beta+(m-1)\lambda_k}$. After direct manipulation, this is equivalent to showing that 
$$(m-1)^2(\lambda_1-\lambda_k)^2-2(m-1)\beta(\lambda_1-\lambda_k)+\beta^2\geq 0,$$
which is true for all $m$ since the left-hand side is a complete square: $\left((m-1)(\lambda_1-\lambda_k)-\beta\right)^2$.
\end{proof}

Given these two simple lemmas, observe that 
\begin{align}
g^*(1)&=g_1(1)\leq g_1(m)^{\frac{1}{m}}=g^*(m)^{\frac{1}{m}},\quad \text{ for all} ~m\in [1, \frac{\beta}{\lambda_1-\lambda_k}+1],\label{put-lemmas-tog1}
\end{align}
where the first and last equalities follow from the fact that $g_1(m)=g^*(m)$ for $m \in [1, \frac{\beta}{\lambda_1-\lambda_k}+1]$, and the second inequality follows from Lemma~\ref{lem:proof-opt-batch1}. Also, observe that 
\begin{align}
g^*(1)&=g_1(1)\leq g_1(m)^{\frac{1}{m}}\leq g_2(m)^{\frac{1}{m}} = g^*(m)^{\frac{1}{m}},\quad \text{ for all} ~m > \frac{\beta}{\lambda_1-\lambda_k}+1,\label{put-lemmas-tog2}
\end{align}
where the third inequality follows from Lemma~\ref{lem:proof-opt-batch2}, and the last equality follows from the fact that $g_2(m)=g^*(m)$ for $m > \frac{\beta}{\lambda_1-\lambda_k}+1.$ Putting (\ref{put-lemmas-tog1}) and (\ref{put-lemmas-tog2}) together, we have $g^*(1)\leq g^*(m),$ for all $m\geq 1$, which completes the proof.


\section{Experiments: Comparison of Train and Test losses}
\label{sec:expr-test}
\begin{figure}[!ht]
  \vspace{-3mm}
  \centering 
  \begin{minipage}[b]{0.28\textwidth}
    \includegraphics[width=\textwidth]
{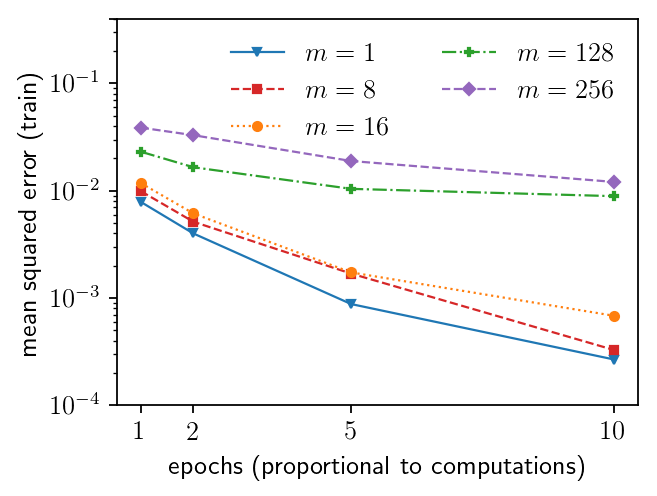}  
    \subcaption*{~~~~~~MNIST (train)}
  \end{minipage}
  \hfill
  \begin{minipage}[b]{0.28\textwidth}
    \includegraphics[width=\textwidth]
{fig/mbs/hint_10k_elap_ep_tr.png}
    \subcaption*{~~~~~~HINT-S (train)}
  \end{minipage}
  \hfill
  \begin{minipage}[b]{0.28\textwidth}
    \includegraphics[width=\textwidth]
{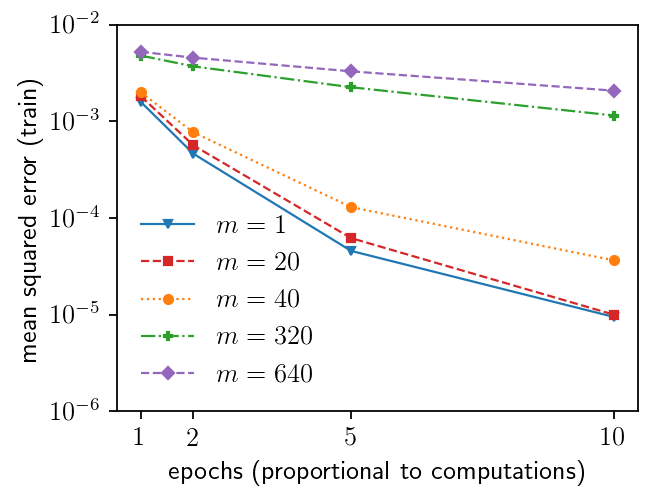}
    \subcaption*{~~~~~~TIMIT (train)}
  \end{minipage}
  \\ 
   \begin{minipage}[b]{0.28\textwidth}
    \includegraphics[width=\textwidth]
{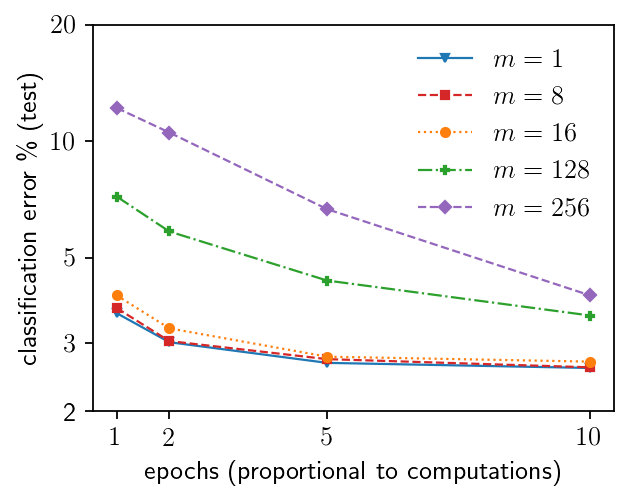}
    \subcaption{MNIST (test)}
  \end{minipage}
  \hfill
  \begin{minipage}[b]{0.28\textwidth}
    \includegraphics[width=\textwidth]
{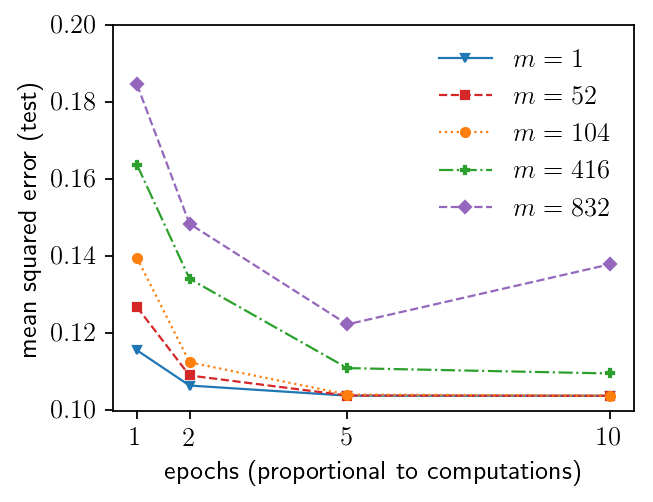}
    \subcaption{HINT-S (test)}
  \end{minipage}
  \hfill
  \begin{minipage}[b]{0.28\textwidth}
    \includegraphics[width=\textwidth]
{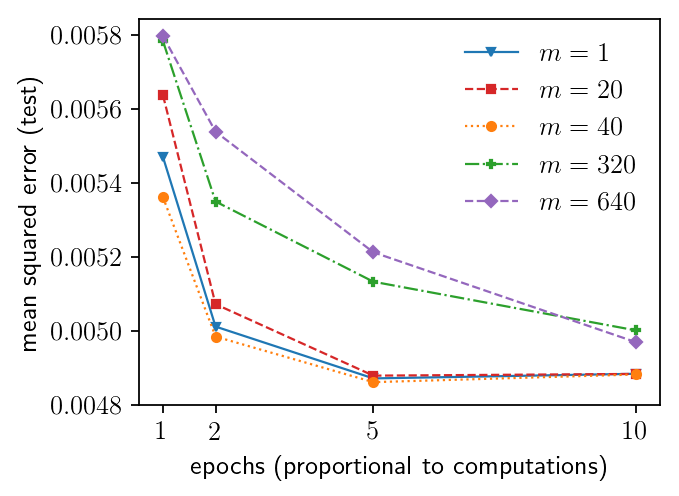}
    \subcaption{TIMIT (test)}
  \end{minipage}
  \caption{Comparison of training error ($n = 10^4$) and testing error for different mini-batch sizes ($m$) vs. number of epochs (proportional to computation, note for $n$ data points, $n \cdot N_{epoch}= m \cdot N_{iter} $)}
  \vspace{-5mm}
\end{figure}

\end{document}